\documentclass[]{article}
\pdfoutput=1
\usepackage[utf8]{inputenc} 
\usepackage[T1]{fontenc}    
\usepackage{booktabs}       
\usepackage{amsfonts}       
\usepackage{nicefrac}       
\usepackage{microtype}      

\usepackage{graphicx}
\usepackage{amsthm}	
\usepackage{amsmath}
\usepackage{amssymb}
\usepackage{url}
\usepackage{flushend}
\usepackage{relsize}
\usepackage{nomencl}

\makenomenclature

\newtheorem{definition}{Definition}
\newtheorem{lemma}{Lemma}
\newtheorem{assumption}{Assumption}
\newtheorem{theorem}{Theorem}
\newtheorem{remark}{Remark}

\graphicspath{{Graphics/}}

\makeatletter
\newcommand*{\iso}[1]{{\hbox{$\left#1\vbox to7.5\p@{}\right.\n@space$}}}
\makeatother

\begin{document}

\title{Force-Sensor-Less Bilateral Teleoperation Control of Dissimilar Master-Slave System with Arbitrary Scaling}

\author{Santeri Lampinen$^1$\footnote{Address all correspondence to this author; santeri.lampinen@tuni.fi} \and Janne Koivumäki$^1$ \and Wen-Hong Zhu$^2$ \and Jouni Mattila$^1$}
	
\date{\footnotesize{$^1$Faculty of Engineering and Natural Sciences, Tampere University, Finland}\\
	  \footnotesize{$^2$Canadian Space Agency,  Longueuil (St-Hubert), QC, Canada}}

\maketitle

\begin{abstract}
	This study designs a high-precision bilateral teleoperation control for a dissimilar master-slave system. The proposed nonlinear control design takes advantage of a novel subsystem-dynamics-based control method that allows designing of individual (decentralized) model-based controllers for the manipulators locally at the subsystem level. Very importantly, a dynamic model of the human operator is incorporated into the control of the master manipulator. The individual controllers for the dissimilar master and slave manipulators are connected in a specific communication channel for the bilateral teleoperation to function. Stability of the overall control design is rigorously guaranteed with arbitrary time delays. Novel features of this study include the completely force-sensor-less design for the teleoperation system with a solution for a uniquely introduced computational algebraic loop, a method of estimating the exogenous operating force of an operator, and the use of a commercial haptic manipulator. Most importantly, we conduct experiments on a dissimilar system in 2 degrees of freedom (DOF). As an illustration of the performance of the proposed system, a force scaling factor of up to 800 and position scaling factor of up to 4 was used in the experiments. The experimental results show an exceptional tracking performance, verifying the real-world performance of the proposed concept.
\end{abstract}


\pagebreak
	\nomenclature[1]			{$(\cdot)_\gamma $}									{Subscript indicating whether the attribute refers to the master ($ \gamma = m $) or the slave~($\gamma=s$).}
	\nomenclature[2Theta]		{$ \Theta_{m} \in \mathbb{R}^{30\times2}$}			{Mapping matrix.}
	\nomenclature[1q]			{$ \dot{\mathbf{q}}_m  \in \mathbb{R}^{2} $}		{Independent joint velocity coordinates of the master manipulator.}
	\nomenclature[1Vm]			{$ \mathcal{ V }_\gamma  \in \mathbb{R}^{2}$}		{Velocity of the master/slave manipulator.}
	\nomenclature[1Vmr]			{$ \mathcal{ V }_{\gamma r}  \in \mathbb{R}^{2}$}	{Required counterpart of $ \mathcal{ V }_\gamma $.}
	\nomenclature[1Vmd]			{$ \mathcal{ V }_{\gamma d}  \in \mathbb{R}^{2}$}	{Desired  counterpart of $ \mathcal{ V }_\gamma $.}
	\nomenclature[1Pm]			{$ \mathcal{ P }_\gamma  \in \mathbb{R}^{2}$}		{Control point position of the manipulator.}
	\nomenclature[1Jm]			{$ \mathbf{J}_m \in \mathbb{R}^{2\times2} $}		{Jacobian matrix.}
	\nomenclature[2Phi]			{$ \Phi_m  \in \mathbb{R}^{30\times2} $}			{Mapping matrix.} 	
	\nomenclature[1Mm]			{$ M_m^* \hspace{-0.1cm} \in \hspace{-0.05cm} \mathbb{R}^{30\times30} $}				{Equivalent inertial matrix.}
	\nomenclature[1Cm]			{$ C_m^* \hspace{-0.1cm} \in \hspace{-0.05cm} \mathbb{R}^{30\times30} $}				{Skew-symmetric matrix of the centrifugal and Coriolis terms.}
	\nomenclature[1Gm]			{$ G_m^* \in \mathbb{R}^{30} $}						{Gravitation vector.}
	\nomenclature[2taum]		{$ \boldsymbol{\tau}_m \in \mathbb{R}^{2}$}			{Applied torques of the master manipulator.}
	\nomenclature[2taumm]		{$ \boldsymbol{\tau}_{mm} \in \mathbb{R}^{2}$}		{Estimated dynamics of the master manipulator.}
	\nomenclature[1fgamma]		{$ \mathbf{ f }_\gamma \in \mathbb{R}^{2} $}		{Contact force of the master/slave manipulator.}
	\nomenclature[1A]			{$ \mathbf{A} \in \mathbb{R}^{2\times2} $}			{Diagonal positive-definite matrix defining the gain of the force-feedback.}
	\nomenclature[1C]			{$ \mathbf{C} \in \mathbb{R}^{2\times2} $}			{Diagonal positive-definite matrix defining the time-constant of the first-order filter.}
	\nomenclature[2Lambda]		{$ \mathbf{\Lambda} \in\mathbb{R}^{2\times2}$}		{Diagonal positive-definite matrix defining the gain of the position-feedback.}
	\nomenclature[1Km]			{$ K_m \in \mathbb{R}^{2\times2} $}					{Diagonal positive-definite matrix defining the gain of the internal velocity-feedback.}
	\nomenclature[2kappa f]		{$ \kappa_{ f } \in \mathbb{R} $}					{Force scaling factor.}
	\nomenclature[2kappa p]		{$ \kappa_{ p } \in \mathbb{R} $}					{Position scaling factor.}
	\nomenclature[1T]			{$ T \in \mathbb{R} $}								{Length of the one-way time delay.}
	\nomenclature[2sigma]		{$ \sigma_f $}										{Selective factor to detect contact motion.}

\printnomenclature[1.7cm]
{\small 
	Tilde (${}^{\sim}$) on top of a variable implies that the variable is filtered with a first-order filter, unless explicitly specified otherwise. Hat ($\hat{\phantom{a}}$)}~on~top of a variable implies that the variable is an estimate of itself.

\section{Introduction}
\label{sec:intro}

Bilaterally teleoperated robotic systems can bring~the perception and precision of direct manipulation~into~challenging and risk-intensive tasks in environments that may be hazardous or hostile for humans. In contrast to unilateral teleoperation where the command flow goes only from the master to the slave, bilateral teleoperation provides the operator with information about the slave manipulator in the form of force feedback, to assist in the coordination and decision-making processes. To broaden the application scope of teleoperation, arbitrary motion and forces scaling has been pursued by many researchers, but no rigorously stability guaranteed method have been shown to work in a multi-DOF system.

Currently, one of the most interesting applications for teleoperation lies in Learning from Demonstrations (LdD) applications with heavy-duty manipulators. LfD is an established technique in robotics, where a robot is taught to perform tasks by demonstrations from a human teacher. The robot can then repeat these tasks in even slightly varying conditions \cite{argall2009survey}. The \textit{key enabler} for LfD applications with \textbf{heavy-duty manipulators} is teleoperation of asymmetric systems with~\textit{motion} and \textit{force scaling}. Conventional kinesthetic teaching methods, an established method for providing teaching samples, cannot be applied for such manipulators due to the size and force limitations (workspace over 2 m and payload over 500~kg) \cite{Suomalainen2018}. Instead, teaching samples can be captured using teleoperation with motion and force scaling between~the manipulators. Teleoperation has the advantage of an intuitive and efficient communication and operation strategy between humans and robots. Teleoperated demonstrations have~been successfully used for LfD applications with promising results using 1:800 force scaling in \cite{Suomalainen2018} and using 1:1 scaling in \cite{PervezRyuLfD}, and \cite{LfDSurvey2018}.

In applications where heavy objects are handled or a~great amount of force is required, hydraulic actuation has remained the most attractive solution due to its great power-to-weight ratio. Hydraulic actuators further have the benefits of simplicity, robustness, and low cost. However, control of such actuators is significantly challenged by their complex nonlinear dynamic behavior. When the actuators are used in articulated systems, the control design is further complicated by the associated nonlinear multi-body dynamics, and the overall dynamics can be described by coupled nonlinear third-order differential equations. Consequently, the constrained motion control of multiple degrees-of-freedom (\textit{n}-DOF) hydraulic robotic manipulators has been a well-recognized challenge \cite{Mattila2017}. 

As an additional challenge to the above, contact~force~measurements are often required for contact control. In conventional applications, a 6-DOF force/torque sensor is often attached to the tip of the manipulator for this purpose. However, these force/torque sensors are expensive and prone to overloading and shocks, a situation frequently occurring with hydraulic heavy-duty manipulators \cite{Koivumaki_TRO2015}. Therefore, methods avoiding direct contact force measurements have become~desirable.

Needless to say, teleoperation of hydraulic manipulators has been an extremely difficult problem due to unresolved challenges in their high-precision control \cite{HOKAYEM20062035, OSTOJASTARZEWSKI1989, Salcudean1999, Tafazoli2002}. However, due to recent advances in hydraulic manipulators' high-precision control and leaps in the state-of-the-art (see \cite{Mattila2017, KoivumakiTMECH2017,Koivumaki_TRO2015,KoivumakiCEP2019}), teleoperation of hydraulic manipulators is suddenly becoming a feasible and interesting field of study again. Moreover, \textit{time delay}, a focused research topic, especially in the teleoperation of extraterrestrial systems\cite{ZhaiTeleop2018}, can be alternatively addressed for terrestrial applications in the advent of 5G cellular networks with ultra-low latencies \cite{Aijaz2017}. Terrestrial applications are within the author's main scope. 

In this paper, we target an asymmetric bilateral teleoperation system comprised of a commercial haptic master manipulator and a hydraulic slave manipulator. The system has notable asymmetry between the manipulators due to substantial differences between the dynamics of the master and slave manipulators. Due to this asymmetry, handling motion and force scaling in the teleoperation architecture becomes necessary. The current state-of-the-art in teleoperation control has been focusing on purely electrical manipulators in symmetrical configurations, including multi-master or multi-slave setups, shared control, or dealing with time delays \cite{Guo2018, MalyszSirouspour2, Sirouspour2009, Sirouspour2005}. The existing methods for teleoperation of hydraulic manipulators have mainly relied on linear control theory and system linearization \cite{OSTOJASTARZEWSKI1989, Salcudean1999, Tafazoli2002, Muhammad2007}. However, these methods have limitations in teleoperation of complex, highly nonlinear, and asymmetric systems. In contrast, the adaptive teleoperation scheme proposed by Zhu and Salcudean in \cite{Zhu2000} was reported to be capable of addressing nonlinear dynamics of asymmetric master and slave manipulators with arbitrary motion and force scaling. However, experiments with only a 1-DOF symmetrical system were presented. Moreover, both manipulators were equipped with force sensors.

In the present study, the results of \cite{Zhu2000} and \cite{LampinenCASE} are used as the foundation for designing high-precision bilateral teleoperation control for significantly asymmetric systems. In \cite{LampinenCASE}, preliminary attempts for full-dynamics-based (and high-precision) bilateral teleoperation for an asymmetric hydraulic/electric system were demonstrated, while in \cite{LampinenFPMC2018} artificial constraints in the task space were implemented. However, sufficient stability analysis and theoretical discussions were not included. 

To improve the preliminary theory and control performance reported in \cite{LampinenCASE}, the following \textit{distinguishable contributions} are demonstrated in the present study.   
\textbf{1)} We propose a master manipulator contact force estimation by using joint control torques and estimated manipulator dynamics. A solution~to~a computational algebraic loop formed around the actuation and force estimation is proposed. 
\textbf{2)} We propose a novel method for estimating the exogenous force of the human~operator.
\textbf{3)} Stability of the overall control design~is rigorously guaranteed with robustness against an arbitrary time delay.

With the control theoretical developments described above, the experiments demonstrate significant improvements in relation to our preliminary study \cite{LampinenCASE}. The experiments with a 2-DOF system with a force scaling ratio of up to 800 and a position scaling ratio of up to 4, in lieu of the 1-DOF experiments~in \cite{Zhu2000}, serve\textit{ a critical step toward practical 6-DOF applications}.

The rest of this paper is organized as follows: Section \ref{sec:Math} presents the mathematical preliminaries. Section \ref{sec:Master} discusses control of the master manipulator, while Section \ref{sec:Slave} discusses control of the slave manipulator. Section \ref{sec:teleoperation} presents the teleoperation scheme and discusses properties of the teleoperation method. Section \ref{sec:Experiments} presents the experimental system and results. Finally, conclusions are drawn in Section~\ref{sec:Conclusions}.

\section{Mathematical Preliminaries}
\label{sec:Math}

Let an orthogonal coordinate system (i.e., a frame) $\{\mathbf{A}\}$~be attached to a rigid body. Then, the linear/angular velocity~vector $^{\mathbf{A}}{V} \in \mathbb{R}^6$ and the force/moment vector $^{\mathbf{A}}{F} \in \mathbb{R}^6$ of the rigid body, expressed in frame $ \{\mathbf{A}\} $, can be expressed as \cite{Zhu2010Virtual}:
\begin{equation*}
	^{\mathbf{A}}{V} = 
	\begin{bmatrix}
	^{\mathbf{A}}\mathbf{v} & {}^{\mathbf{A}}\boldsymbol{\omega}
	\end{bmatrix}^T,
	\hspace{0,1cm}
	^{\mathbf{A}}{F} = 
	\begin{bmatrix}
	^{\mathbf{A}}\mathbf{f} & {}^{\mathbf{A}}\mathbf{m}
	\end{bmatrix}^T
\end{equation*}
where $ ^{\mathbf{A}}\mathbf{v} \in \mathbb{R}^3 $ and $ ^{\mathbf{A}}\boldsymbol{\omega} \in \mathbb{R}^3 $ are the linear and angular velocity vectors of frame~$ \{\mathbf{A}\} $, expressed in frame~$ \{\mathbf{A}\} $, and $ ^{\mathbf{A}}\mathbf{f} \in \mathbb{R}^3 $ and $ ^{\mathbf{A}}\mathbf{m} \in \mathbb{R}^3 $ are the force and moment vectors that are being measured and expressed in frame $ \{\mathbf{A}\} $. 

Transformation of linear/angular and force/moment vectors between two frames, attached to a common rigid body, namely $ \{\mathbf{A}\} $ and $ \{\mathbf{B}\} $, can be expressed as \cite{Zhu2010Virtual}:
\begin{gather}
	\label{eq:velocitytransformation}
	^{\mathbf{B}}{{V}} = {}^{\mathbf{A}}\mathbf{U}_{\mathbf{B}}^{T}\,{}^{\mathbf{A}}{{V}}\\
	\label{eq:forcetransformation}
	^{\mathbf{A}}{{F}} = {}^{\mathbf{A}}\mathbf{U}_{\mathbf{B}}\,{}^{\mathbf{B}}{{F}}
\end{gather}
where $^{\mathbf{A}}\mathbf{U}_{\mathbf{B}} \in \mathbb{R}^{6\times6}$ is a force/moment transformation~matrix that also transform velocities between frames $\{\bf{A}\}$~and~$\{\bf{B}\}$.

The dynamics of a freely moving rigid body, expressed in the fixed rigid body frame $ \{\mathbf{A}\} $, can be defined as
\begin{equation}
	\label{eq:rigidbodydynamics}
	\mathbf{M_A}\frac{d}{dt}(^{\mathbf{A}}\mathbf{{V}})+\mathbf{C_A}(^{\mathbf{A}}\omega)^{\mathbf{A}}\mathbf{{V}}+\mathbf{G_A} = {}^{\mathbf{A}}{{F}}^{*}
\end{equation}
where $\mathbf{M_A} \in \mathbb{R}^{6\times6} $ is the mass matrix, $\mathbf{C_A} \in \mathbb{R}^{6\times6} $ the Coriolis and centrifugal terms, and $\mathbf{G_A \in \mathbb{R}^{6}}$ the gravity vector of the rigid body. For a detailed formulation of $\mathbf{M_A} $, $\mathbf{C_A} $ and $\mathbf{G_A} $, readers are referred to \cite{Zhu2010Virtual}.

\section{The Master Manipulator}
\label{sec:Master}
Phantom premium 3.0/6DOF, a commercial haptic manipulator without any modifications to the hardware, has been chosen to act as the master manipulator in this study. It possesses 6-DOF manipulability and force-feedback along each individual DOF, with a workspace mimicking human arm motion pivoting from the shoulder. For this study, we developed a new control system for the manipulator to rigorously address the dynamics of the lightweight manipulator. As a challenge to the control design, the manipulator lacks force/torque sensors. Therefore, human operator contact force estimation is required. Without loss of generality, we consider manipulation and force perception only within a specific 2-DOF plane (by using joint 2 and joint 3), while the rest of the DOFs are locked with software. Benefits of using a commercial haptic device as the master manipulator is the ease of implementation for wide range of applications. 

Frame $ \{ \mathbf{S_{tcp}} \} $ is assigned to the tool center point (TCP)~of the master manipulator (see Fig.~\ref{fig:FrameAssignment}) and the external forces~resulting from the dynamics of the human operator as well~as the exogenous operating hand force are estimated and expressed~in this frame. Orientation of frame $ \{\mathbf{S_{tcp}}\} $ is aligned with handle of the master manipulator and the handle is held horizontal as shown in Fig.~\ref{fig:FrameAssignment} by position control at the spherical wrist.

\subsection{Concept of the Virtual Decomposition Control}
\label{sec:VDC_Concept}
To design the intended high-precision teleoperation~for~complex asymmetric system, the study takes advantage of a novel virtual decomposition control~(VDC)~approach (see \cite{Zhu2010Virtual,Zhu1997}). The method is developed especially~for~controlling complex robotic systems, with a number of significant state-of-the-art control performance improvements with robotic systems (see, e.g., \cite{Zhu2000,Zhu2013,Koivumaki_TRO2015,Mattila2017,KoivumakiTMECH2017,LampinenCASE,KoivumakiCEP2019}). As a key feature, VDC enables to \textit{virtually} break down complexity of the original system to a set of manageable \textit{modular subsystems} \cite{Zhu2010Virtual,Zhu2013} such that the control design and stability analysis can be performed \textit{locally at the subsystem level} without imposing additional approximations. This allows, e.g., that \textit{changing the control (or dynamics) of a subsystem does not affect the control equations of the rest of the system}~\cite{Zhu2010Virtual}. 

The subsystem-dynamics-based control design philosophy in VDC originates from two unique concepts, namely \textit{virtual stability} and \textit{virtual power flows} (VPFs); see Appendix~\ref{App:Virtual_Stab}. The VPFs uniquely define the dynamic interactions among the subsystems such that the \textit{virtual stability} of every subsystem ensures that a positive VPF is connected to its corresponding negative VPF in the adjacent subsystem (and vice versa). Thus, when every subsystem qualify as \textit{virtually stable}, all the VPFs cancel each other out, eventually, leading to the stability of the entire system in the sense of Lebesgue integrable functions (see Appendix~\ref{App:Lebesgue_Stab}). For more detailed information and additional benefits of VDC, see \cite{Zhu2010Virtual,Mattila2017}.

\begin{figure}[t]
	\begin{center}
		\includegraphics[width=0.65\columnwidth]{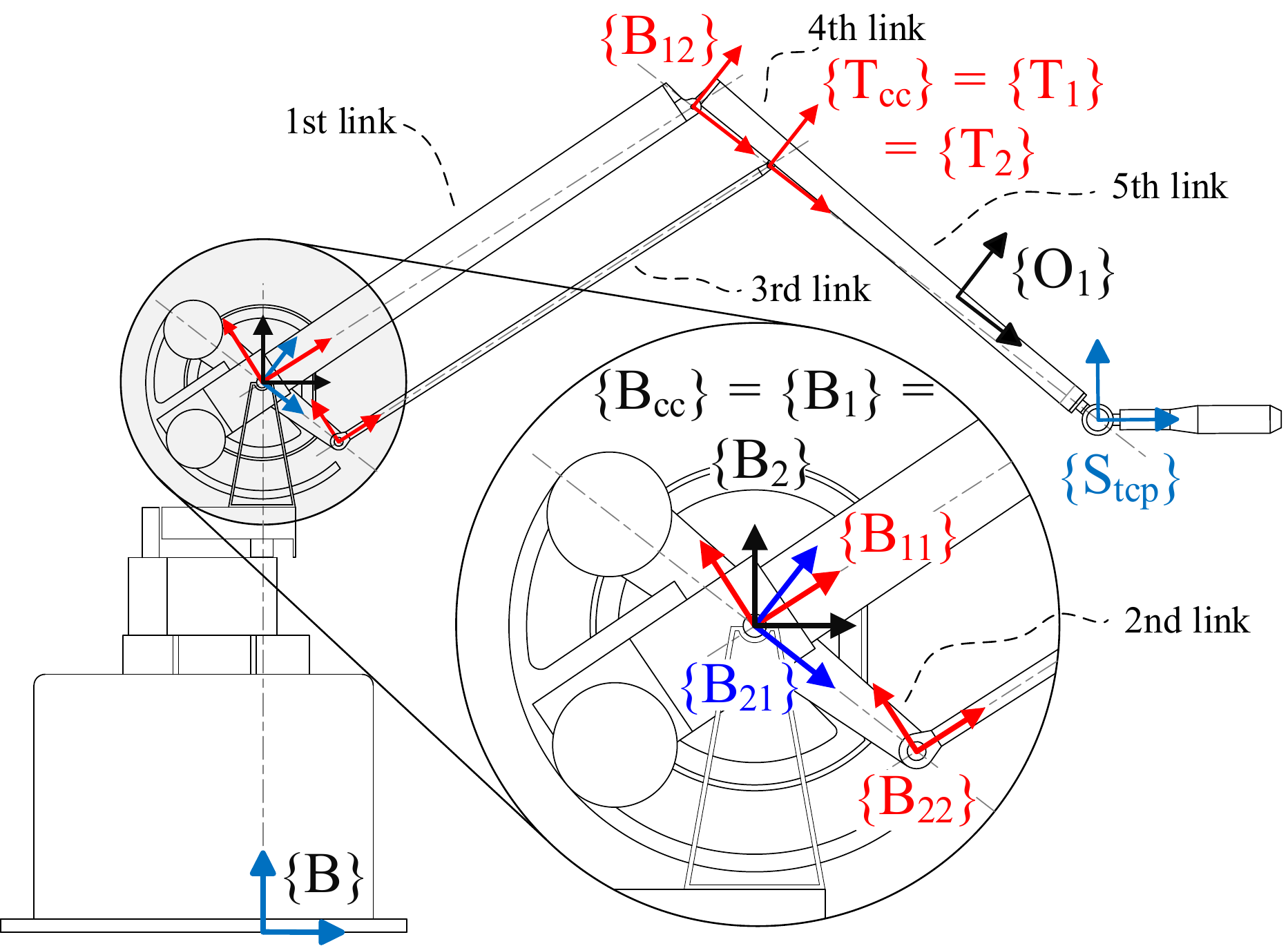}
	\end{center}
	\caption{Frame assignment of the haptic manipulator.}
	\label{fig:FrameAssignment} 
\end{figure}

\begin{figure}[b]
	\begin{center}
		\includegraphics[width=0.5\columnwidth]{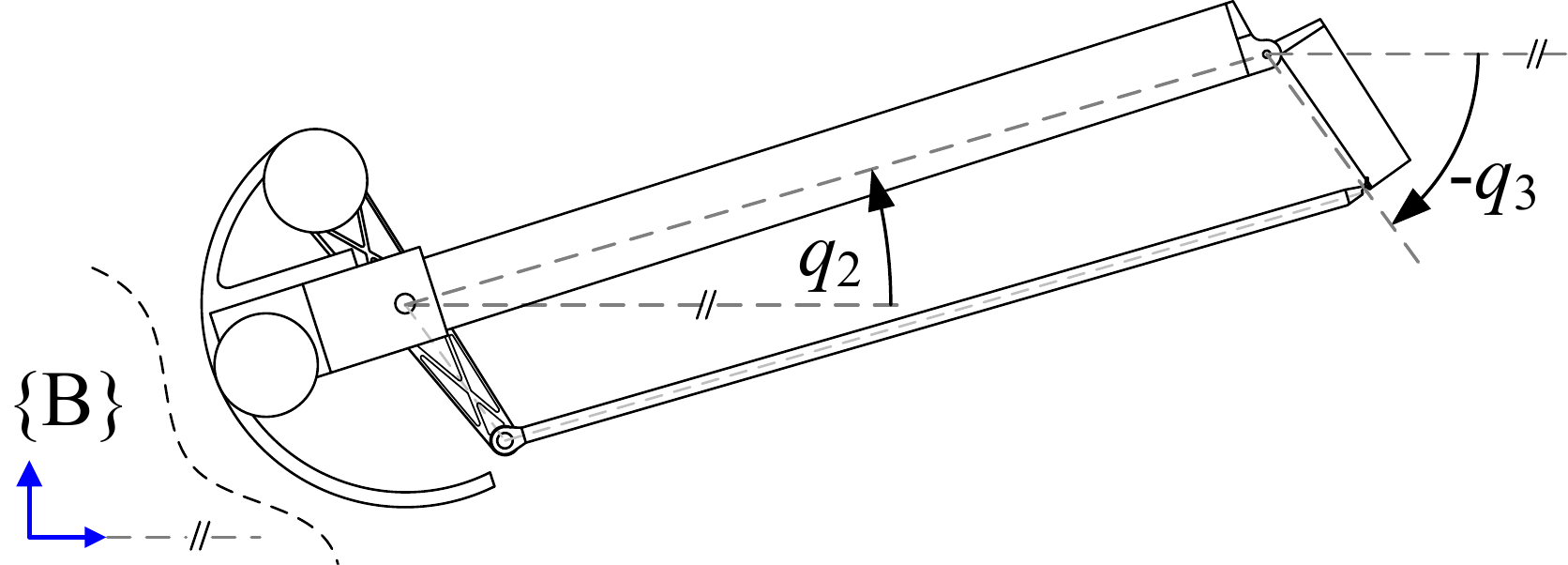}
	\end{center}
	\caption{Illustration of the joint angles of the master manipulator.}
	\label{fig:closedchain1} 
\end{figure}

\subsection{Kinematics}

Relevant coordinate frames in terms of control of the master manipulator are shown in Fig.~\ref{fig:FrameAssignment}. Notably, frame $ \{\mathbf{B}_{11} \} $ is attached to the first link of the manipulator, $ \{\mathbf{B}_{21} \} $ to the second link, $ \{\mathbf{B}_{22} \} $ to the third link, $ \{\mathbf{B}_{12} \} $ to the fourth link, and $ \{\mathbf{O}_{1} \} $ to the fifth link. Numbering of the manipulator links is defined in Fig.~\ref{fig:FrameAssignment}.

\begin{remark}
	Link 4 and 5 are virtually cut from the same rigid body, using design principles of the VDC approach, which allows separate computations later.
\end{remark}

The independent joint velocity coordinates are denoted as 
\begin{equation}
\dot{\mathbf{q}}_m = 
\begin{bmatrix}
\dot{q}_{2} & \dot{q}_3
\end{bmatrix}^T \in \mathbb{R}^2 .
\end{equation}
The respective joint angles ${q}_{2}$ and ${q}_{3}$ are shown in Fig.~\ref{fig:closedchain1}. Then, velocities of all links of the master manipulator can~be determined using the geometrical transformation matrices between each frame with the independent joint velocity coordinates as
\begin{equation}
\label{eq:kinematicsV2}
\mathbf{V}_m
=
\Theta_m \dot{\mathbf{q}}_m
\end{equation}
where $ \mathbf{ V }_m = 
\left[\begin{smallmatrix}
^{\mathbf{B_{11}}}{{V}}^T &
^{\mathbf{B_{12}}}{{V}}^T &
^{\mathbf{B_{21}}}{{V}}^T &
^{\mathbf{B_{22}}}{{V}}^T &
^{\mathbf{O_{1}}}{{V}}^T
\end{smallmatrix}\right]^T \in \mathbb{R}^{30} $, and $ \Theta_{m} \in \mathbb{R}^{30\times2}$ is a mapping matrix defined as
\begin{align}
\label{eq:phi1}
&\Theta_m \hspace{-0.1cm} = \hspace{-0.1cm}
\left[
\begin{smallmatrix}
\mathbf{z}	&	\boldsymbol{0}_{6\times1}	\\
^{\mathbf{B_{11}}}\mathbf{U}_{\mathbf{B_{12}}}^{T}\mathbf{z} - \mathbf{z}	&	\mathbf{z}	\\
\boldsymbol{0}_{6\times1}	&	\mathbf{z}	\\
\mathbf{z}	& ^{\mathbf{B_{21}}}\mathbf{U}_{\mathbf{B_{22}}}^{T}\mathbf{z} - \mathbf{z}\\
^{\mathbf{T_{cc}}}\mathbf{U}_{\mathbf{O_{1}}}^{T}{}^{\mathbf{B_{12}}}\mathbf{U}_{\mathbf{T_{cc}}}^{T}\left( ^{\mathbf{B_{11}}}\mathbf{U}_{\mathbf{B_{12}}}^{T}\mathbf{z} - \mathbf{z} \right) \,	& ^{\mathbf{T_{cc}}}\mathbf{U}_{\mathbf{O_{1}}}^{T}{}^{\mathbf{B_{12}}}\mathbf{U}_{\mathbf{T_{cc}}}^{T}\mathbf{z}	
\end{smallmatrix}
\right] \hspace{-0.1cm}
\end{align}
where $ \mathbf{z}  = \left[\begin{smallmatrix} 0 & 0 & 0 & 0 & 0& 1 \end{smallmatrix}\right]^T \in \mathbb{R}^6 $.

Let the independent velocity coordinates at master manipulator's handle be $  \mathcal{V}_m \in \mathbb{R}^2 $ subject to  
\begin{align}
	\mathcal{V}_m &= \mathbf{J}_m \dot{\mathbf{q}}_m \\
	\mathbf{J}_m &= 
		\left[
			\begin{smallmatrix}
				\boldsymbol{I}_{2\times 2} &
				\boldsymbol{0}_{2\times 4}
			\end{smallmatrix}
		\right]
	{}^{\mathbf{O_{1}}}\mathbf{U}_{\mathbf{S_{tcp}}}^{T}
	\left[
		\begin{smallmatrix}
			\boldsymbol{0}_{6\times 24}&\boldsymbol{I}_{6\times 6}
		\end{smallmatrix}
	\right]
	\Theta_m
\end{align}
where $ \mathbf{J}_m \in \mathbb{R}^{2\times2}$ is the invertible Jacobian matrix of the master manipulator.

Then, another mapping matrix can be defined as
\begin{equation}
	\Phi_m = \Theta_{m}\mathbf{J}_m^{-1}.
\end{equation}

\subsection{Dynamics}
\label{sec:dynamics}

Dynamics of each rigid body of the master manipulator can be determined using \eqref{eq:rigidbodydynamics} and \eqref{eq:kinematicsV2}. According to \cite{Zhu2008dynamics}, and under Assumption~\ref{ass:friction}, dynamic model of the master manipulator can be expressed as
\begin{align}
		&\Phi_m^T \mathcal{M}_m^* \frac{d}{dt}\left(\Phi_m\mathcal{V}_m\right) + \left(\Phi_m^T \mathcal{C}_m^* \Phi_m\right) \mathcal{V}_m + \Phi_m^T \mathcal{G}_m^*
		\label{eq:dynamicsofMaster}
		= \mathbf{J}_m^{-T}\boldsymbol{\tau}_m - \mathbf{f}_m
\end{align}
where $ \mathbf{f}_m $ is the net reaction force from the master manipulator toward the human operator and will be defined later in more detail, $ \boldsymbol{\tau}_m \in \mathbb{R}^2 $ denotes the applied torques of the manipulator, and
\begin{align}
	\label{eq:Mm*}
	\mathcal{M}_m^* &= {\rm diag} \left\{\mathbf{M}_{\mathbf{B_{11}}}, \ {}\mathbf{M}_{\mathbf{B_{12}}}, \ {}\mathbf{M}_{\mathbf{B_{21}}}, \ {}\mathbf{M}_{\mathbf{B_{22}}}, \	{}\mathbf{M}_{\mathbf{O_{1}}}\right\}\\
	\mathcal{C}_m^* &= {\rm diag} \left\{\mathbf{C}_{\mathbf{B_{11}}}, \ {}\mathbf{C}_{\mathbf{B_{12}}}, \ {}\mathbf{C}_{\mathbf{B_{21}}}, \ {}\mathbf{C}_{\mathbf{B_{22}}}, \	{}\mathbf{C}_{\mathbf{O_{1}}}\right\}\\
	\mathcal{G}_m^* &= \left[\mathbf{G}_{\mathbf{B_{11}}}^T, \ {}\mathbf{G}_{\mathbf{B_{12}}}^T, \ {}\mathbf{G}_{\mathbf{B_{21}}}^T, \ {}\mathbf{G}_{\mathbf{B_{22}}}^T, \	{}\mathbf{G}_{\mathbf{O_{1}}}^T\right]^T.
	\label{eq:Gm*}
\end{align}

\begin{assumption}
	\label{ass:friction}
	Bearing friction of all the revolute joints of the master manipulator are zero.
\end{assumption}

\subsection{Human Operator}

Based on literature \cite{Kazerooni1994}; \cite{Zhu2000}, and \cite{cooke1979}, sufficient accuracy for modeling the human operator can be achieved using a simple second-order linear time-invariant model. The following model is used here
\begin{equation}
	\label{eq:LTImodel}
	\mathbf{M}_h \ddot{\mathbf{x}}_h + \mathbf{D}_h \dot{\mathbf{x}}_h + \mathbf{K}_h {\mathbf{x}}_h = \mathbf{f}_m - \mathbf{f}_h^*
\end{equation}
where $ \mathbf{M}_h \in \mathbb{R}^{2\times2} $, $ \mathbf{D}_h \in \mathbb{R}^{2\times2} $ and $ \mathbf{K}_h \in \mathbb{R}^{2\times2} $ are symmetric positive-definite matrices approximating the inertia, damping and stiffness of the arm of the human operator, respectively; while $ \mathbf{f}_m \in \mathbb{R}^2 $, appeared first in \eqref{eq:dynamicsofMaster}, denotes the net force vector, exerted by the master manipulator toward the operator, and $ \mathbf{f}_h^* \in \mathbb{R}^2 $ denotes the exogenous force vector actively generated by the operator. The position of the arm of the operator is denoted by $ \mathbf{x}_h \in \mathbb{R}^2 $, while $ \dot{\mathbf{x}}_h \in \mathbb{R}^2 $ and $ \ddot{\mathbf{x}}_h \in \mathbb{R}^2 $ denote the first and second time-derivatives of the position vector, respectively, subject to
\begin{equation}
	\label{eq:x_hDef}
	\dot{\mathbf{x}}_h = \mathcal{V}_m.
\end{equation}

In \cite{MalyszSirouspour2} and \cite{MalyszSirouspour}, it was suggested that the exogenous force of the operator could be estimated using a fast parameter adaptation function. This differs from the approach in \cite{Zhu2000}, where a switching term with a constant force, instead of an estimate, was used to ensure stability. The precise expression of the exogenous force, denoted $ \mathbf{{{f}}}_h^{*} \in \mathbb{R}^2 $ in \eqref{eq:LTImodel}, would necessarily involve research on complex human motor neuron actions. In this paper, we describe this exogenous force as a general linear-in-parameter form as

\begin{equation}
\label{eq:exogenousForceDynamics}
\mathbf{{{f}}}_h^{*} = \Psi(t)\mathbf{p}
\end{equation}
where $ \Psi(t) $ is a time-variant matrix and $ \mathbf{p} $ is a parameter vector. We treat vector $ \mathbf{p} $ as constant by moving all time-variant properties into $ \Psi(t) $.

\begin{remark}
	Note that expression \eqref{eq:exogenousForceDynamics} is quite general. It covers the expressions used in \cite{MalyszSirouspour2} and \cite{Sirouspour2009}, in which $ \Psi(t) = 1 $~are used. Most importantly, this expression takes the same form commonly used in neural networks, allowing flexible incorporation of basis radial functions into machine learning mechanisms. With more elegant design of $ \Psi(t) $, for example, muscle activation measured by electromyography could be used for the intent force modeling.
\end{remark}

\subsection{Control of the Master Manipulator with a Human Operator}

For accurate control of the master manipulator, dynamics of both the manipulator itself and the human operator need to be addressed together. The required control law must therefore define required contact force towards the human operator.

The estimated human operator exogenous force is written~as
\begin{equation}
\label{eq:estimateExogenousForceDynamics}
\hat{\mathbf{{{f}}}}_h^{*} = \Psi(t)\hat{\mathbf{p}}
\end{equation}
where $ \hat{\mathbf{p}} $ is an estimate of the parameter vector. The time-invariant parameters are estimated using the following parameter adaptation law as
\begin{equation}
	\label{eq:standardAdaptation}
	\begin{split}
	\dot{\hat{p}}_{i} &= \rho_{i} \, \kappa \, \Psi_{i}(t) \, s(t) \\
	s(t) &= \left(\mathcal{V}_{mr} - \mathcal{V}_{m}\right)\\
	\kappa &= \left\{\begin{matrix} 0, \quad {\hat{p}}_{i} \le {\hat{p}}_{i}^- \ \rm{and} \ s \le 0 \\
									0, \quad {\hat{p}}_{i} \ge {\hat{p}}_{i}^+ \ \rm{and} \ s \ge 0 \\
									1, \quad \rm{otherwise} \phantom{asdasdaa}	\end{matrix}	\right.
	\end{split}
\end{equation}
where $ \hat{p}_{i} $, $ {\hat{p}}_{i}^+ $ and $ {\hat{p}}_{i}^- $ are the estimate of the $ {i} $th element of the real time-invariant parameter vector $ \mathbf{p}=\left[p_1, \, p_2,\,...,\,p_{i}, \,...\right]^T $ as well as its upper and lower bounds, respectively, $ \rho_{i} $ is the adaptation gain of the $ {i} $th element of $ \hat{\mathbf{p}} =\left[\hat{p}_1, \, \hat{p}_2,\,...,\,\hat{p}_{i}, \,...\right]^T  $, $ \Psi_{i}(t) $ denotes the $ {i} $th column of the time-variant matrix $ \Psi(t) $, and $ \mathcal{V}_{mr} \in \mathbb{R}^2 $ denotes the required velocity at the tip of the master manipulator, expressed in frame $ \mathbf{\{S_{tcp}\}}$.

In Section~\ref{sec:dynamics}, dynamics were calculated using the measured independent joint velocity vector $ \dot{\mathbf{q}}_m $. However, since the proposed control method is velocity based, we need to define the required velocities in Cartesian space. Let $ \mathcal{V}_{md} \in \mathbb{R}^2 $ be the desired velocity of the tip of the master manipulator, to be defined later in Section~\ref{sec:teleoperation}. Then, the required velocity vector, $ \mathcal{V}_{mr} \in \mathbb{R}^2 $, is designed as
\begin{equation}
\label{eq:VmrDef}
\mathcal{V}_{mr} = \mathcal{V}_{md} - \mathbf{A}\tilde{{\mathbf{f}}}_m
\end{equation}
where $ \mathbf{A} \in \mathbb{R}^{2\times2} $ is a diagonal positive-definite gain matrix, and $ \tilde{{\mathbf{f}}}_m $ denotes a filtered estimate of the forces of the master manipulator to be determined later in this section. Compute
\begin{align}
	\label{eq:qmr}
	\dot{\mathbf{q}}_{mr} &= \mathbf{J}_m^{-1}\mathcal{V}_{mr} \\
	\label{eq:vmr}
	\mathbf{ V }_{mr} &= \Theta_m\dot{\mathbf{q}}_{mr}
\end{align}
where $ \dot{\mathbf{q}}_{mr} \in \mathbb{R}^2 $ denotes the required counterpart of $ \dot{\mathbf{q}}_{m} $ and $ \mathbf{V}_{mr} \in \mathbb{R}^{30} $ denotes the required counterpart of $ \mathbf{V}_{m} $.
\begin{remark}
	The second term in right hand side of \eqref{eq:VmrDef} acts as a local force feedback term within the control design.
\end{remark}

The linear parametrization of the required rigid body dynamics can be written according to \cite{Zhu2010Virtual} as
\begin{equation}
\label{eq:linearparametrization}
\mathbf{Y_A \theta_A} \equiv
\mathbf{M_A}\frac{d}{dt}(^{\mathbf{A}}V_{\rm r})+\mathbf{C_A}(^{\mathbf{A}}\omega)^{\mathbf{A}}V_{\rm r}+\mathbf{G_A}.
\end{equation}
Interested reader is referred to the formulation of the regressor matrix $\mathbf{Y_A} \in \mathbb{R}^{6\times13}$ and the parameter vector $\mathbf{\boldsymbol{\theta}_A} \in \mathbb{R}^{13}$ in \cite{Zhu2010Virtual}.

Using $ \mathbf{ A } \in \{\mathbf{B_{11}}, \mathbf{B_{12}}, \mathbf{B_{21}}, \mathbf{B_{22}}, \mathbf{O_{1}}\} $, dynamics of each rigid body can be calculated with \eqref{eq:linearparametrization} as
\begin{align}
\label{eq:dynamicsofmasterprecalc}
\mathbf{Y}_m \boldsymbol{\theta}_m = & \nonumber
\left[ 
\left(\mathbf{Y_{B_{11}} \theta_{B_{11}}}\right)^T, \ \left(\mathbf{Y_{B_{12}} \theta_{B_{12}}}\right)^T, \ \left(\mathbf{Y_{B_{21}} \theta_{B_{21}}}\right)^T, \right. \\ 
&\ \,  \left(\mathbf{Y_{B_{22}} \theta_{B_{22}}}\right)^T, \ \left.\left(\mathbf{Y_{O_{1}} \theta_{O_{1}}}\right)^T \right]^T \in \mathbb{R}^{30}.
\end{align}
Furthermore, dynamics of the human operator are calculated with a similar linear parametrization form as
\begin{equation}
	\label{eq:reqOperatorDynamics}
	\mathbf{Y}_h \boldsymbol{\theta}_h = \mathbf{M}_h \dot{\mathcal{V}}_{m \rm r} + \mathbf{D}_h \dot{\mathbf{x}}_h + \mathbf{K}_h {\mathbf{x}}_h.
\end{equation}

Then, control equations for the master manipulator can be defined as
\begin{equation}
\mathbf{J}_{m}^{-T}\boldsymbol{\tau}_{m} = \Phi_{m}^T\mathbf{Y}_m \boldsymbol{\theta}_m + \mathbf{Y}_h \boldsymbol{\theta}_h 
		   + \hat{\mathbf{f}}_h^* +  K_m\left({\mathcal{V}}_{mr}  - {\mathcal{V}}_{m}\right)
		   \label{eq:controlofmaster}
\end{equation}
where $ K_m \in \mathbb{R}^{2\times2} $ is a positive-definite gain matrix. The last term in \eqref{eq:controlofmaster} is a velocity feedback term used to ensure the control stability.

\subsection{Force Estimation}

The net reaction force from the master manipulator toward the operator can be estimated using the known dynamics of the master manipulator as a base. This method is similar to the inverse dynamics based estimation methods, described in \cite{Haddadin2017}. The main difference here is that the estimated actuator torque and applied torque are calculated based on the inverse dynamics, yielding that the external force can be estimated in addition to the mere collision detection, possible with the simpler method. The force estimate can be expressed as
\begin{align}
	\label{eq:fmEstimation}
	\mathbf{\hat{f}}_m = & \mathbf{J}_m^{-T} \left(
	\boldsymbol{\tau}_{m} - \boldsymbol{\tau}_{mm} \right)
\end{align}
where $ \boldsymbol{\tau}_{m} $ is the master robot control input defined in \eqref{eq:controlofmaster}, and $ \boldsymbol{\tau}_{mm} $ is the estimated master robot dynamics, defined as
\begin{align}
	\boldsymbol{\tau}_{mm} = &\left(\Theta_m^T \mathcal{M}_m^* \Theta_m\right)\hat{ \ddot{\mathbf{q}}}_m + \big(\Theta_m^T \mathcal{C}_m^* \Theta_m+\Theta_m^T \mathcal{M}_m^* \dot{\Theta}_m\big) \hat{ \dot{\mathbf{q}}}_m + \Theta_m^T \mathcal{G}_m^*
\end{align}
where $ \hat{ \ddot{\mathbf{q}}}_m $ and $  \hat{ \dot{\mathbf{q}}}_m $ are estimates of $ { \mathbf{\ddot{q}}}_m $ and $ { \mathbf{\dot{q}}}_m $, respectively, obtained by differentiation from the measured joint angles $ \mathbf{q}_m $. The filtered estimate of the master manipulator force vector is obtained using
\begin{equation}
\label{eq:filteredFm}
\dot{\tilde{\mathbf{f}}}_m + \mathbf{C}\tilde{\mathbf{f}}_m = \mathbf{C}\hat{\mathbf{f}}_m
\end{equation}
where $ \mathbf{C} \in \mathbb{R}^{2\times2}$ is a diagonal positive definite matrix.

\subsection{Computation Algorithms}

Differentiating \eqref{eq:VmrDef} and expressing $ \dot{\mathcal{V}}_{mr} $ as an affine function of $ \dot{\tilde{\mathbf{f}}}_m $ yields
\begin{equation}
\label{eq:dotVmrAffine}
\dot{\mathcal{V}}_{mr} = \mathbf{ A }_1(t) \dot{\tilde{\mathbf{f}}}_m + \mathbf{B}_1(t)
\end{equation}
where $ \mathbf{ A }_1(t) \in \mathbb{R}^{2\times2} $ is a known matrix and $ \mathbf{ B }_1(t) = \dot{\mathcal{V}}_{md} \in \mathbb{R}^2$ is a known vector (which will be given in Section~\ref{sec:teleoperation}), and $ \dot{\tilde{\mathbf{f}}}_m \in \mathbb{R}^2 $ is a vector to be specified later in this subsection.

Using \eqref{eq:qmr} and \eqref{eq:vmr}, it follows from \eqref{eq:linearparametrization}--\eqref{eq:fmEstimation}, that
\begin{equation}
	\label{eq:AffineFunction2}
	\mathbf{\hat{f}}_m =  \mathbf{J}_m^{-T}\left(\boldsymbol{\tau}_m - \boldsymbol{\tau}_{mm}\right) = \mathbf{ A }_2(t)\mathbf{ A } \dot{\tilde{\mathbf{f}}}_m + \mathbf{ B }_2(t)
\end{equation}
where $ \mathbf{ A }_2(t) \in \mathbb{R}^{2\times2} $ is a known matrix and $ \mathbf{ B }_2(t) \in \mathbb{R}^2$ is a known vector. Then, it follows from \eqref{eq:filteredFm}
\begin{equation}
	\dot{\tilde{\mathbf{f}}}_m = \left(\mathbf{C}\mathbf{ A }_2(t)\mathbf{ A }\right) \dot{\tilde{\mathbf{f}}}_m + \mathbf{ B }_3(t)
	\label{eq:forceEstimationDerivateBasis}
\end{equation}
where $ \mathbf{ B }_3(t) \in \mathbb{R}^2$ is a known vector. The existence of a computational algebraic loop can be clearly seen in \eqref{eq:forceEstimationDerivateBasis}.

To ensure numerical stability, we must have
\begin{equation}
	\substack{\mathlarger{\sigma}\\\scriptsize \rm max}	\left(\mathbf{C}\mathbf{ A }_2(t)\mathbf{ A }\right) < 1.
\end{equation}
This means both $ \mathbf{ C } $ and $ \mathbf{ A } $ must be restricted.

Finally, $ \dot{\tilde{\mathbf{f}}}_m $ can be computed from \eqref{eq:forceEstimationDerivateBasis} as
\begin{equation}
	\label{eq:dTildeFm}
	\dot{\tilde{\mathbf{f}}}_m = \left[\mathbf{ I }_{2\times2} - \left(\mathbf{C}\mathbf{ A }_2(t)\mathbf{ A }\right)\right]^{-1}\mathbf{ B }_3(t).
\end{equation}
Once $ \dot{\tilde{\mathbf{f}}}_m $ is obtained, $ {\tilde{\mathbf{f}}}_m $ in \eqref{eq:VmrDef} can be computed using integration with $ \dot{\tilde{\mathbf{f}}}_m(0) = 0 $.

\subsection{Stability}
Substituting \eqref{eq:controlofmaster} and \eqref{eq:LTImodel} into \eqref{eq:dynamicsofMaster} yields
\begin{align}
	\label{eq:StabilityBaseFunction}
	\Phi_m^T &\mathcal{M}_m^* \frac{d}{dt}\Big(\Phi_m\left(\mathcal{V}_m - \mathcal{V}_{mr}\right)\Big) + M_h\left(\dot{\mathcal{V}}_m -  \dot{\mathcal{V}}_{mr}\right) \nonumber \\
	& = \left(\Phi_m^T \mathcal{C}_m^* + K_m\right)\left(\mathcal{V}_{mr} - \mathcal{V}_m\right) + \Psi(t)\left(\hat{\mathbf{p}}-\mathbf{p}\right).
\end{align}
Then the non-negative function for the master manipulator is chosen as
\begin{align}
	\nu_m = \ &\frac{1}{2}\left(\mathcal{V}_{mr}- \mathcal{V}_m\right)^T\left(\Phi_m^T \mathcal{M}_m^* \Phi_m + M_h\right)\left(\mathcal{V}_{mr} - \mathcal{V}_m\right)
	+ \frac{1}{2}\sum_{i}\frac{p_{i} - \hat{p}_{i}}{\rho_{i}}.
	\label{eq:nonnegfunctionMaster}
\end{align}
The time-derivative of the non-negative function in \eqref{eq:nonnegfunctionMaster} is obtained using \eqref{eq:StabilityBaseFunction}, \eqref{eq:standardAdaptation} and the skew-symmetric properties of $ \mathcal{C}_m^* $ as
\begin{equation}
	\label{eq:dnonnegfunctionMaster}
	\dot{\nu}_m \le - \left(\mathcal{V}_{mr}- \mathcal{V}_m\right)^T K_m \left(\mathcal{V}_{mr}- \mathcal{V}_m\right).
\end{equation}

\begin{theorem}
	Analyzing the master manipulator \eqref{eq:dynamicsofMaster} with the human operator \eqref{eq:LTImodel} subject to control \eqref{eq:controlofmaster} with estimated~exogenous operator force using adaptation law \eqref{eq:standardAdaptation}, it yields
	\begin{equation}
	\label{eq:convergenceMaster}
	\xi_m \equiv \mathcal{V}_{md} - \mathcal{V}_{m} - \mathbf{A}\tilde{{\mathbf{f}}}_m \in L_2 \bigcap L_\infty.
	\end{equation}
\end{theorem}

The proof directly follows \eqref{eq:nonnegfunctionMaster} and \eqref{eq:dnonnegfunctionMaster}. For the concept of $L_2$ \textit{and} $L_\infty$ \textit{stability} (having similarities to \textit{Lyapunov functions method}), see Appendix~\ref{App:Lebesgue_Stab}.

\section{Slave Manipulator}
\label{sec:Slave}

A commercial HIAB-031 hydraulic manipulator is chosen~to act as the slave manipulator of the teleoperation system. The manipulator is retrofitted with fast hydraulic servo valves,~pressure transducers to measure cylinder chamber pressures and high accuracy incremental encoders to measure joint angles. Although the manipulator is retrofitted, it does not have force/ torque sensor at the TCP. Consequently, a force-sensor-less control method with external force estimation is used~for the slave manipulator as was the case with the master manipulator.  

In the experiments, manipulation and force perception is considered within the same 2-DOF plane as with the master manipulator. The extension cylinder and rotation of the boom was mechanically locked. Fig.~\ref{fig:HiabSOG} a illustrates the slave manipulator and shows several important frames of the manipulator. Frame $ \{\mathbf{B_s}\} $ is fixed to the base of the slave manipulator, frame $ \{\mathbf{O}_2\} $ is attached to the last link of the slave manipulator and frame $ \{\mathbf{G}\} $ is attached to the tip of the slave manipulator and has the same orientation as frame $ \{\mathbf{O}_2\} $. Frame $ \{\mathbf{C}\} $ has the same origin as frame $ \{\mathbf{G}\} $, but is aligned with frame $ \{\mathbf{B}_s\} $.

As discussed in Section~\ref{sec:VDC_Concept}, \emph{VDC enables modularity in the control~design}. Consequently, the slave manipulator can be considered~as a subsystem (with its own local subsystems) of the overall~system. Stability-guaranteed constrained motion control of the manipulator is described in \cite{Koivumaki_TRO2015}. To incorporate the control system designed in \cite{Koivumaki_TRO2015}, control equations of the last object need to be adjusted while \emph{rest of the control system is kept identical to that of} \cite{Koivumaki_TRO2015}. Fig.~\ref{fig:HiabSOG} presents the decomposed structure of the slave manipulator, with the re-used control design circled by a dashed line.

\begin{figure}[t]
	\begin{center}
		\includegraphics[width=0.65\columnwidth]{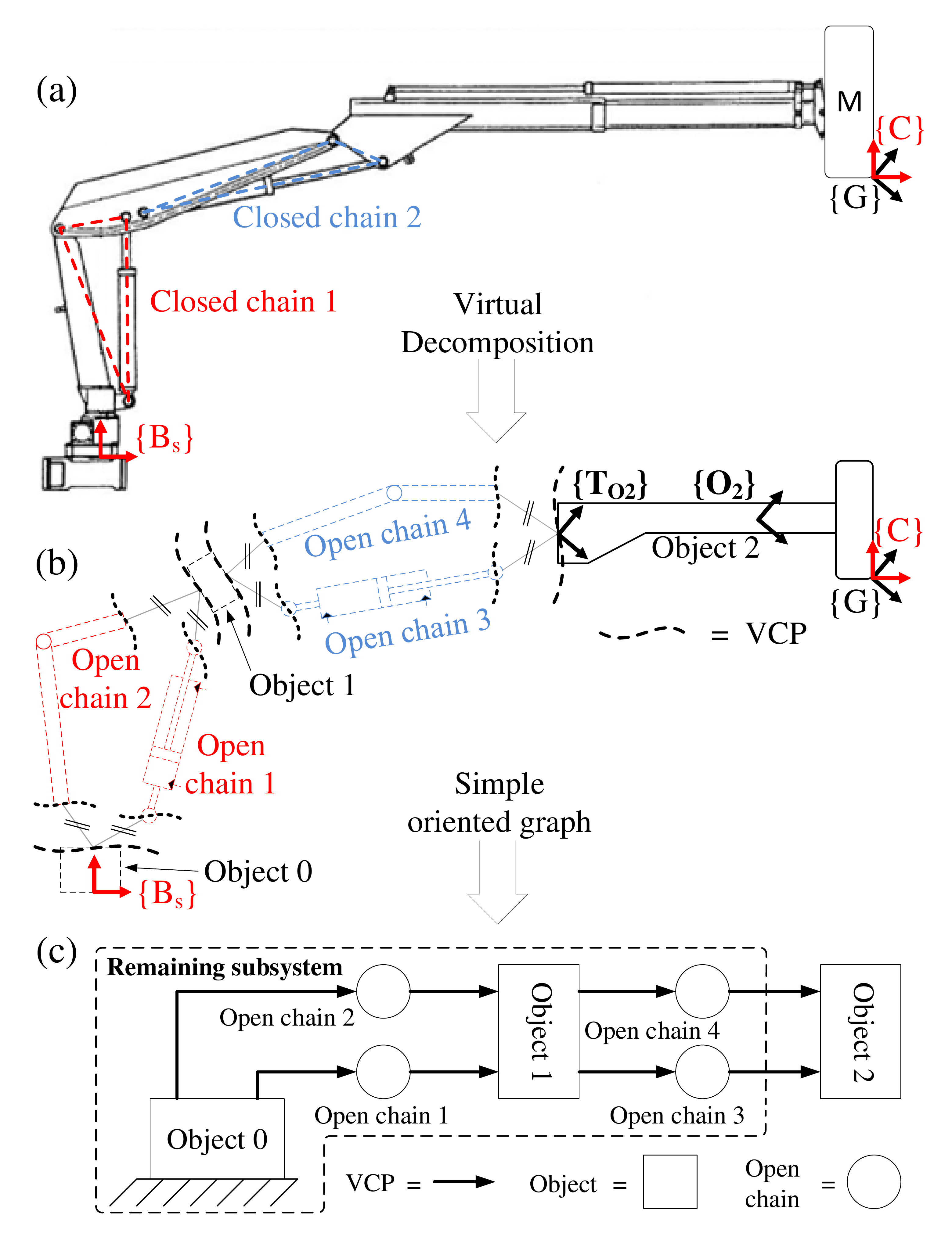}
	\end{center}
	\caption{(a) The slave manipulator, (b) Virtual decomposition of the slave manipulator (c) simple oriented graph (SOG) of the slave manipulator. The circled area in the SOG represents subsystem of the slave manipulator, for which the control has been designed in \cite{Koivumaki_TRO2015}.}
	\label{fig:HiabSOG} 
\end{figure}

\subsection{Object 2 -- Kinematics and Dynamics}

Let the linear/angular velocity vector ${}^{\mathbf{T}_{\rm O2}}V \in \mathbb{R}^6$ at the driven VCP of Object 2 be known from the kinematic chain through the previous subsystems (see \cite{Koivumaki_TRO2015}). Then, kinematic transformations among the frames in Object 2 (see Fig. \ref{fig:object2}) can be written as 
\begin{align}
	{}^{\mathbf G}V_{} &= {}^{\mathbf{T}_{\rm O2}}{\mathbf{U}}^T_{\mathbf G}{}^{\mathbf{T}_{\rm O2}}V \nonumber \\
										 &= {}^{\mathbf{O}_2}{\mathbf{U}}^T_{\mathbf G}{}^{\mathbf{O}_2}V	\label{EQ_GV} \\
	{}^{\mathbf C}V_{} &={\rm diag}({}^{\mathbf G}{\mathbf R}_{\mathbf C},{}^{\mathbf G}{\mathbf R}_{\mathbf C})	{}^{\mathbf G}V_{} \label{eq:CV}
\end{align}

Next, dynamics of the environment are defined. In this work, we assume flexible environment with dynamics described by second-order linear time-invariant model \cite{Zhu2000} as
\begin{equation}
	\label{eq:fs}
	\mathbf{f}_{s} = 
	\mathbf{M}_e\ddot{\mathbf{x}}_{s} + \mathbf{D}_e \dot{\mathbf{x}}_{s} + \mathbf{K}_e \mathbf{x}_{s}
\end{equation}
where $\mathbf{M}_e \in \mathbb{R}^{2\times2}$, $\mathbf{D}_e \in \mathbb{R}^{2\times2}$ and $\mathbf{K}_e \in \mathbb{R}^{2\times2}$~are~symmetric positive-definite matrices approximating the inertia,~damping and stiffness of the environment, respectively, and $\mathbf{x}_s \in \mathbb{R}^2$ denotes the tip position of the slave manipulator, expressed in frame $ \{\mathbf{B_s}\} $, subject to $\dot{\mathbf{x}}_{s} = \mathcal{ V }_s$~with
\begin{equation} \label{EQ_Vs}
	\mathcal{V}_s = \left[\mathbf{I}_{2\times2}\ \ \boldsymbol{0}_{2\times4}\right]{}^{\mathbf { C }}{ V }.
\end{equation}
Then, dynamics of the environment can be included on the slave manipulator as
\begin{align}
	{}^{{\mathbf G}}F_{} &= {\rm diag}({}^{\mathbf G}{\mathbf R}_{\mathbf C},{}^{\mathbf G}{\mathbf R}_{\mathbf C}) 
	\left[\mathbf{I}_{2\times2}\ \ \boldsymbol{0}_{2\times4}\right]^T
	\sigma_f \mathbf{f}_{s} \label{EQ_GF}
\end{align}
where
\begin{equation} \label{EQ_switching}
	\sigma_{ f} = \left\{
	\begin{array}{l l}
	0 & \quad \text{approach motion}\\
	1 & \quad \text{constrained motion.}
	\end{array} \right.
\end{equation}

The net force/moment vector (rigid body dynamics) ${}^{{\mathbf O}_{2}}F^{{\mathbf *}}$ of Object 2 can be written in view of \eqref{eq:rigidbodydynamics} as
\begin{equation} \label{EQ_O2F*}
{{\mathbf M}}_{{\mathbf O}_{2}}\frac{d}{dt}({}^{{\mathbf O}_{2}}V)+{{\mathbf C}}_{{\mathbf O}_{2}}({}^{{\mathbf O}_{2}}{\omega }){}^{{\mathbf O}_{2}}V+{{\mathbf G}}_{{\mathbf O}_{2}}={}^{{\mathbf O}_{2}}F^{*}.
\end{equation}
and, eventually, the force balance (i.e. force resultant) equation of Object 2 can be written as
\begin{equation} \label{EQ_O2F*2}
{}^{{\mathbf O}_{2}}F^{{\mathbf *}} = {{}^{{\mathbf O}_{2}}{{\mathbf U}}}_{{{\mathbf T}}_{{\rm o}2}}{}^{{{\mathbf T}}_{{\rm o}2}}F - \sigma_{f}{{}^{{\mathbf O}_{2}}{{\mathbf U}}}_{\mathbf G}{}^{\mathbf G}F
\end{equation}

\subsection{Object 2 -- Control}

Let the required velocity of the slave manipulator be designed as 
\begin{equation}
	\label{eq:VsrDef}
	\mathcal{V}_{sr} = \mathcal{V}_{sd} - \mathbf{A}\tilde{\mathbf{f}}_s
\end{equation}
where $ \mathcal{V}_{sd} \in \mathbb{R}^2 $ is to be defined in Section~\ref{sec:teleoperation}, and $ \tilde{\mathbf{f}}_s \in \mathbb{R}^2 $ is obtained from $ \hat{\mathbf{f}}_s $ using a first order filter as
\begin{equation}
	\dot{\tilde{\mathbf{f}}}_s + \mathbf{C}\tilde{\mathbf{f}}_s = \mathbf{C}\hat{\mathbf{f}}_s
\end{equation}
and $ \hat{\mathbf{f}}_s $ is obtained using (15) in \cite{Koivumaki_TRO2015}.

Required piston velocities of the slave manipulator are then redesigned from (87) in \cite{Koivumaki_TRO2015} into
\begin{equation} \label{EQ_xr_dots}
{\begin{bmatrix}
	\dot{x}_{\rm 1r} \\
	\dot{x}_{\rm 3r}
	\end{bmatrix}} = \mathbf{J}^{-1}_{x}\mathcal{ V }_{sr}
\end{equation}
where $ \mathbf{J}^{-1}_{x} \in \mathbb{R}^{2\times2} $ is the invertible Jacobian matrix of the slave manipulator, defined in \cite{Koivumaki_TRO2015}.

Then, in view of \eqref{EQ_GV}, \eqref{eq:CV} and \eqref{EQ_Vs}, the required linear/angular velocity vectors in Object 2 can be written as 
\begin{align}
	{}^{\mathbf G}V_{r} &= {}^{\mathbf{T}_{\rm O2}}{\mathbf{U}}^T_{\mathbf G}{}^{\mathbf{T}_{\rm O2}}V_r \nonumber \\
		&= {}^{\mathbf{O}_2}{\mathbf{U}}^T_{\mathbf G}{}^{\mathbf{O}_2}V_r \label{EQ_GVr} \\
	{}^{\mathbf C}V_{r} &={\rm diag}({}^{\mathbf G}{\mathbf R}_{\mathbf C},{}^{\mathbf G}{\mathbf R}_{\mathbf C})	{}^{\mathbf G}V_{r} \label{eq:CVr}\\
	\mathcal{V}_{sr} &= \left[\mathbf{I}_{2\times2}\ \ \boldsymbol{0}_{2\times4}\right]{}^{\mathbf {C}}{V}_r. \label{EQ_Vsr}
\end{align}
	
The required contact force of the slave manipulator is designed as
\begin{equation}
	\label{eq:fsr}
	\mathbf{f}_{sr} = \mathbf{M}_e\dot{\mathcal{V}}_{sr} + \mathbf{D}_e \dot{\mathbf{x}}_{s} + \mathbf{K}_e \mathbf{x}_{s}.
\end{equation}
Finally, using \eqref{eq:linearparametrization} and \eqref{EQ_GF}--\eqref{EQ_O2F*2} the required control laws for Object 2 dynamics can be written as 
\begin{align}	
	{}^{{\mathbf G}}F_{r} &= {\rm diag}({}^{\mathbf G}{\mathbf R}_{\mathbf C},{}^{\mathbf G}{\mathbf R}_{\mathbf C})\left[\mathbf{I}_{2\times2}\ \ \boldsymbol{0}_{2\times4}\right]^T	\sigma_f \mathbf{f}_{sr} \label{EQ_GFr}\\
	{}^{{\mathbf O}_{2}}F^{*}_r &= {\mathbf Y}_{{\mathbf O}_2}{\widehat{\pmb\theta}}_{{\mathbf O}_2}+{{\mathbf K}}_{{{\mathbf O}_2}}({}^{{{\mathbf O}_2}}V_{\rm r} - {}^{{{\mathbf O}_2}}V) \label{EQ_O2F*r} \\
	{}^{{\mathbf O}_{2}}F^{*}_r &= {{}^{{\mathbf O}_{2}}{{\mathbf U}}}_{{{\mathbf T}}_{{\rm o}2}}{}^{{{\mathbf T}}_{{\rm o}2}}F_r - \sigma_{f}{{}^{{\mathbf O}_{2}}{{\mathbf U}}}_{\mathbf G}{}^{\mathbf G}F_r \label{EQ_O2F*r2}
\end{align}
In line with \eqref{eq:linearparametrization}, ${\mathbf Y}_{{\mathbf O}_2}{\widehat{\pmb\theta}} \in \mathbb{R}^{6}$ in \eqref{EQ_O2F*r} is the model-based feedforward compensation term for the rigid body dynamics and ${\mathbf K}_{{\mathbf O}_2} \in \mathbb{R}^{6\times6}$ is a positive-definite velocity feedback matrix to ensure the control stability. By defining
\begin{align}
	{\mathbf{s}}_{{\mathbf O}_{2}} &= {\mathbf Y}^T_{{\mathbf O}_{2}}({}^{{\mathbf O}_{2}}{V_{\rm r}}-{}^{{\mathbf O}_{2}}V) \label{PA_sO2}
\end{align}
the estimated parameter vector ${\widehat{\pmb\theta}}_{{\mathbf O}_2} \in \mathbb{R}^{13}$ in \eqref{EQ_O2F*r} is updated~as
\begin{align}
	\dot{\widehat{\theta}}_{{\mathbf O}_2i} &= {\rho}_{i} {{s}}_{{\mathbf O}_{2}i} \kappa_i,\ \forall i\in\{1,2,...,13\}  \label{PA_thetaO2} \\
	\kappa_i &= \left\{\begin{matrix} 0, \quad {\widehat{\theta}}_{{\mathbf O}_2i} \le {{\theta}}_{{\mathbf O}_2i}^- \ {\rm and} \ {s}_{{\mathbf O}_{2}i} \le 0 \\
									0, \quad {\widehat{\theta}}_{{\mathbf O}_2i} \ge {{\theta}}_{{\mathbf O}_2i}^+ \ {\rm and} \ {s}_{{\mathbf O}_{2}i} \ge 0 \\
									1, \quad \rm{otherwise} \phantom{asdfdfgfgfgasd}	\end{matrix}	\right.\label{PA_kappaO2}
\end{align}
where $\widehat{\theta}_{{\mathbf O}_2i}$ is the \textit{i}th element of ${\widehat{\pmb\theta}}_{{\mathbf O}_2}$; ${s}_{{\mathbf O}_{2}i}$ is the \textit{i}th element of ${\mathbf{s}}_{{\mathbf O}_{2}}$; ${\rho}_{i} > 0$ is the update gain; ${\theta}_{{\mathbf O}_2i}^-$ is the lower bound of $\widehat{\theta}_{{\mathbf O}_2i}$; and ${{\theta}}_{{\mathbf O}_2i}^+$ is the upper bound of $\widehat{\theta}_{{\mathbf O}_2i}$.

\subsection{Stability}

The remaining system, for which the control was designed first in \cite{Koivumaki_TRO2015}, qualifies as virtually stable according to Theorem~\ref{thm:RemainingSystem}

\begin{theorem}
	\label{thm:RemainingSystem}
	Consider the system encircled by a dashed line in Fig.~\ref{fig:HiabSOG}. The subsystem qualifies virtually stable with its affiliated vector $ \left(^{\mathbf{A}}V_r - ^{\mathbf{A}}V\right), \forall \mathbf{A} \in \Psi_r $ and its affiliated scalar variables $ \left(f_{{\rm p}i{\rm r}} -f_{{\rm p}i}\right) $ for the hydraulic cylinder i, $ \forall i \in {1,2} $, where $ \Psi_r $ contains rigid body frames of each rigid link and object of the remaining subsystem. A non-negative accompanying function for this system can be found as	
	\begin{align}
	\nu_{\mathbf{R}} \ge \ \ & \frac{1}{2} \sum_{\boldsymbol{A} \in \Psi_{\mathrm{r}}}\left(^{\mathbf{A}} V_{\mathrm{r}}-{}^{\mathbf{A}} V\right)^{T} \mathbf{M}_{\mathbf{A}}\left(^{\mathbf{A}} V_{\mathrm{r}}-^{\mathbf{A}} V\right) \nonumber \\ 
	+ &\frac{1}{2} \sum_{i=1}^{2}\left[\frac{1}{\beta k_{\mathrm{xi}}}\left(f_{{\rm p}i{\rm r}} -f_{{\rm p}i}\right)^{2}\right]
	\label{eq:SlaveNonNeg1}
	\end{align}
	such that
	\begin{align}
	\dot{\nu}_{\mathbf{R}} \leqslant &-\sum_{\boldsymbol{A} \in \Psi_{\mathrm{r}}}\left(^{\mathbf{A}} V_{\mathrm{r}}-{}^{\mathbf{A}} V\right)^{T} \mathbf{K}_{\mathbf{A}}\left(^{\mathbf{A}} V_{\mathrm{r}}-{}^{\mathbf{A}} V\right)-p_{\mathbf{T}_{\mathrm{O} 2}} \nonumber \\ &-\frac{k_f}{k_x}\sum_{i=1}^{2}\left(f_{{\rm p}i{\rm r}} -f_{{\rm p}i}\right)
	\label{eq:SlaveDNonNeg1}
	\end{align}
\end{theorem}

\begin{proof}
	The proof for Theorem~\ref{thm:RemainingSystem} can be obtained from the results of \cite{Koivumaki_TRO2015}.
\end{proof}

\begin{theorem} \label{thm:O2}
Consider Object 2 described by \eqref{EQ_GV}--\eqref{EQ_O2F*2}, combined with the control equations \eqref{eq:VsrDef}--\eqref{EQ_O2F*r2} and with the parameter adaptation \eqref{PA_sO2}--\eqref{PA_kappaO2}. This subsystem is virtually stable with its affiliated vector ${}^{{\mathbf O}_{2}}V_{r} - {}^{{\mathbf O}_{2}}V$ being a virtual function in both $L_2$ and $L_\infty$ in the sense of Definition~\ref{def:Vstab}. This is because a non-negative accompanying function
\begin{align}
		\nu_{{{\mathbf O}_2}} &= \frac{1}{2}({}^{{\mathbf O}_{2}}V_{\rm r} - {}^{{\mathbf O}_{2}}V)^T{{\mathbf M}_{{\mathbf O}_{2}}}({}^{{\mathbf O}_{2}}V_{\rm r} - {}^{{\mathbf O}_{2}}V) +\frac{1}{2}\sum_{i=1}^{13}\frac{(\theta_{{{\mathbf O}_2}i} - \widehat{\theta}_{{{\mathbf O}_2}i})^2}{\rho_{{{\mathbf O}_2}i}} \label{STAB_O2_1}
\end{align}
can be found such that
\begin{equation}
		\dot{\nu}_{{{\mathbf O}_2}} \leqslant - ({}^{{\mathbf O}_{2}}V_{\rm r} - {}^{{\mathbf O}_{2}}V)^T{{\mathbf K}}_{{\mathbf O}_{2}}({}^{{\mathbf O}_{2}}V_{\rm r} - {}^{{\mathbf O}_{2}}V) + p_{{\mathbf T}_{{\rm O}2}} - p_{\mathbf G} \label{STAB_O2_2} 
\end{equation}
holds, where
\begin{equation}
		 \int_0^\infty p_{{\mathbf G}}(t)dt \geqslant - \gamma_s \label{STAB_O2_3}
\end{equation}
holds with $0 \leqslant \gamma_s < \infty$. Note that $p_{{\mathbf T}_{{\rm O}2}}$ is the virtual power flow by Definition~\ref{def:VPF} in the driven VCP of Object 2, and $p_{\mathbf G}$ characterizes the virtual power flow between the end-effector and the environment while in constrained motion (i.e., $\sigma_{f}$~=~1).
\end{theorem}

\begin{proof}
See Appendix~\ref{proof:thm_O2}.
\end{proof}

\begin{theorem}
	Considering \eqref{eq:slaveVirtualStbaility} and Definition \ref{def:Vstab}, the contact with the environment qualifies virtually stable. The non-negative accompanying function for the entire slave manipulator can be written by summing the individual functions from \eqref{eq:SlaveNonNeg1} and \eqref{STAB_O2_1} as
	\begin{align}
		\nu_{\rm tot} = &\nu_{\mathbf{R}} + \nu_{{{\mathbf O}_2}} \nonumber \\
		= 
		& \frac{1}{2} \sum_{\boldsymbol{A} \in \Psi_{\mathrm{r}}}\left(^{\mathbf{A}} V_{\mathrm{r}}-{}^{\mathbf{A}} V\right)^{T} \mathbf{M}_{\mathbf{A}}\left(^{\mathbf{A}} V_{\mathrm{r}}-^{\mathbf{A}} V\right) \nonumber \\ 
		+ 
		&\frac{1}{2} \sum_{i=1}^{2}\left[\frac{1}{\beta k_{\mathrm{xi}}}\left(f_{{\rm p}i{\rm r}} -f_{{\rm p}i}\right)^{2}\right] 
		+
		\frac{1}{2}\sum_{i=1}^{13}\frac{(\theta_{{{\mathbf O}_2}i} - \widehat{\theta}_{{{\mathbf O}_2}i})^2}{\rho_{{{\mathbf O}_2}i}} \nonumber \\ 
		+
		&\frac{1}{2}({}^{{\mathbf O}_{2}}V_{\rm r} - {}^{{\mathbf O}_{2}}V)^T{{\mathbf M}_{{\mathbf O}_{2}}}({}^{{\mathbf O}_{2}}V_{\rm r} - {}^{{\mathbf O}_{2}}V)
	\end{align}
	such that
	\begin{align}
		\dot{\nu}_{\rm tot} = 
		&\dot{\nu}_{\mathbf{R}} + \dot{\nu}_{{{\mathbf O}_2}} \\
		\leqslant 
		&-\sum_{\boldsymbol{A} \in \Psi_{\mathrm{r}}}\left(^{\mathbf{A}} V_{\mathrm{r}}-{}^{\mathbf{A}} V\right)^{T} \mathbf{K}_{\mathbf{A}}\left(^{\mathbf{A}} V_{\mathrm{r}}-{}^{\mathbf{A}} V\right)-p_{\mathbf{T}_{\mathrm{O} 2}} \nonumber \\ &-\frac{k_f}{k_x}\sum_{i=1}^{2}\left(f_{{\rm p}i{\rm r}} -f_{{\rm p}i}\right) \nonumber \\
		&- ({}^{{\mathbf O}_{2}}V_{\rm r} - {}^{{\mathbf O}_{2}}V)^T{{\mathbf K}}_{{\mathbf O}_{2}}({}^{{\mathbf O}_{2}}V_{\rm r} - {}^{{\mathbf O}_{2}}V) \nonumber\\
		&+p_{{\mathbf T}_{{\rm O}2}} - p_{\mathbf G}
	\end{align}
	
	Then stability analysis for the remaining subsystems follows exactly as shown in \cite{Koivumaki_TRO2015}, ultimately yielding stability of the entire slave robot. Then it follows that
	\begin{equation}
		\label{eq:convergenceSlave}
		\xi_s \equiv \mathcal{V}_{sd} - \mathcal{V}_{s} - \mathbf{A}\tilde{{\mathbf{f}}}_s \in L_2 \bigcap L_\infty.
	\end{equation}
\end{theorem}

\begin{figure}[b]
	\begin{center}
		\includegraphics[width=0.6\columnwidth]{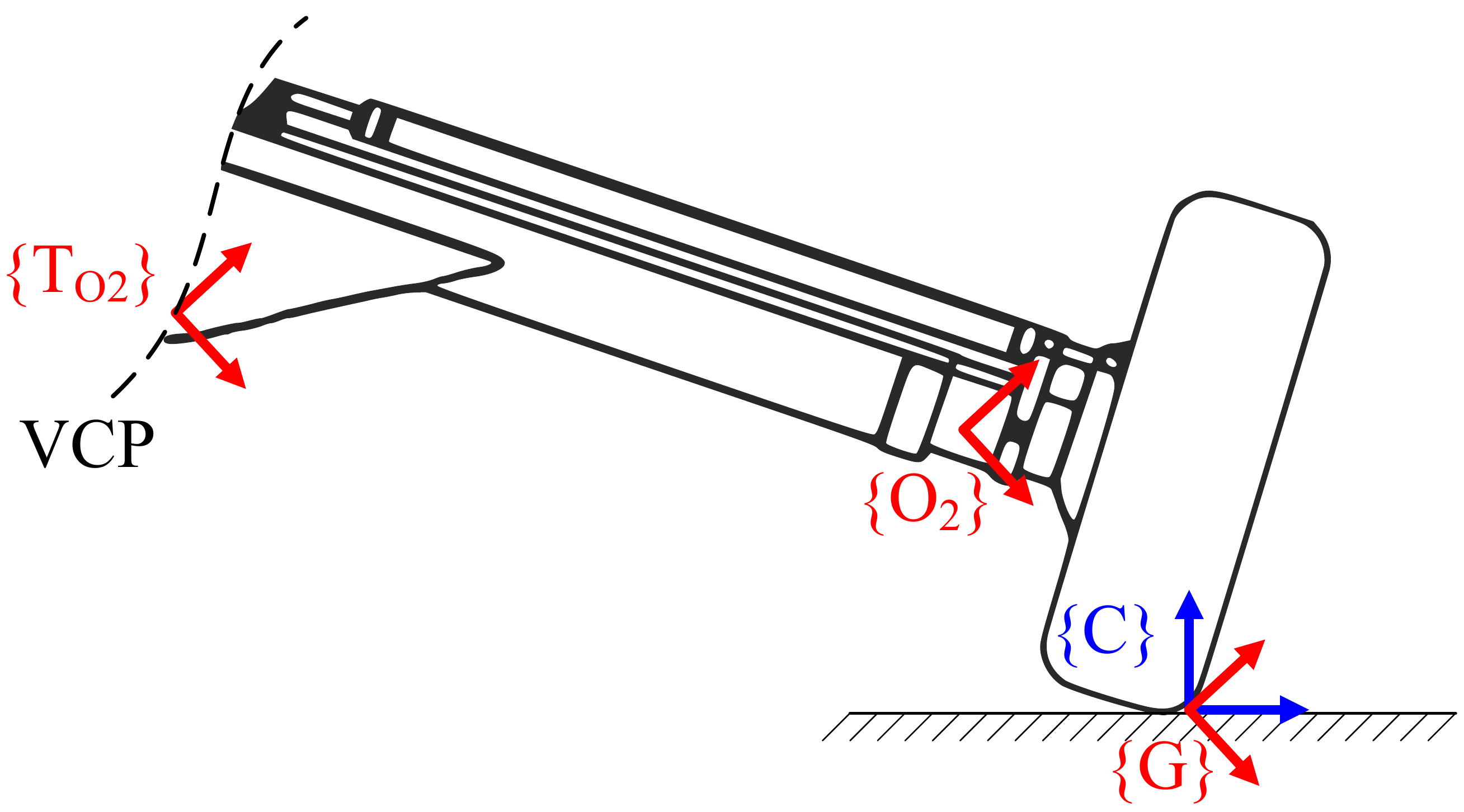}
	\end{center}
	\caption{Contact point with environment.}
	\label{fig:object2} 
\end{figure}

\section{Teleoperation}
\label{sec:teleoperation}

After individual velocity-based controllers for both master and slave manipulators (see \eqref{eq:controlofmaster} and \cite{Koivumaki_TRO2015}) of the teleoperation system have been designed, a scheme for connecting the manipulators can be designed. Connection between the two manipulators is made with a \emph{communication channel}~that~virtually connects the manipulators together. This section designs bilateral teleoperation and specifies two design vectors~$\mathcal{V}_{md}$ and $\mathcal{V}_{sd}$, used in \eqref{eq:VmrDef} and \eqref{eq:VsrDef}, respectively. Using position control $ \delta = 1$ in \cite{Zhu2000}, $\mathcal{V}_{md}$ and $\mathcal{V}_{sd}$ can be designed as
\begin{align}
	\mathcal{V}_{md} = & \frac{1}{\kappa_p}\left[ \tilde{\mathcal{V}}_s + \Lambda\left(\tilde{\mathcal{P}}_s - \kappa_p\mathcal{P}_m\right)  -  \mathbf{A}\left(\tilde{\mathbf{f}}_s + \left({\kappa}_f - {\kappa}_p\right) \tilde{\mathbf{f}}_m\right) \right] \label{eq:teleoperationVmd} \\
	\mathcal{V}_{sd} =  & \kappa_p\tilde{\mathcal{V}}_m - \Lambda\left(\mathcal{P}_s - \kappa_p\tilde{\mathcal{P}}_m\right) 
	- \mathbf{A} {\kappa}_f\tilde{\mathbf{f}}_m 
	\label{eq:teleoperationVsd}
\end{align}
where $ \kappa_p > 0 $ and $ \kappa_f > 0 $ are position and force scaling factors \emph{for arbitrary motion/force scaling between the manipulators}, $ \Lambda \in \mathbb{R}^{2 \times 2}$ is a diagonal positive-definite matrix, and $  \mathcal{P}_m \in \mathbb{R}^{2}$ and  $ \mathcal{P}_s \in \mathbb{R}^{2}$ denote the position/orientation of the master and slave manipulator, respectively, subject to $ \dot{\mathcal{P}}_m = \mathcal{V}_m $ and  $ \dot{\mathcal{P}}_s = \mathcal{V}_s $. Furthermore, $ \tilde{\mathcal{V}}_m $, $ \tilde{\mathcal{V}}_s $, $ \tilde{\mathcal{P}}_m $, $ \tilde{\mathcal{P}}_s $, are filtered values of $ {\mathcal{V}}_m $, $ {\mathcal{V}}_s $, $ {\mathcal{P}}_m $, $ {\mathcal{P}}_s $, respectively, obtained 
using the following first order filter
\begin{equation}
	\label{eq:filter}
	\dot{\tilde{\mathcal{X}}} + \mathbf{C}\tilde{\mathcal{X}} = \mathbf{C}\mathcal{X}
\end{equation}
where $ \mathcal{X} \in \mathbb{R}^2 $ is the input signal and $ \tilde{ \mathcal { X } } \in \mathbb{R}^2 $ is the filtered signal. The use of filtered variables in the two design vectors, \eqref{eq:teleoperationVsd} and \eqref{eq:teleoperationVmd}, makes the required accelerations $ \dot{\mathcal{V}}_{sd} $ and $ \dot{\mathcal{V}}_{md} $, functions of $ \mathcal{V}_{m} $, $ \mathcal{V}_{s} $, $ \mathbf{f}_{m} $ and $ \mathbf{f}_{s} $.

\subsection{Tracking}

Substituting \eqref{eq:teleoperationVmd} and \eqref{eq:teleoperationVsd} into \eqref{eq:convergenceMaster} and \eqref{eq:convergenceSlave}, then subtracting the resulting error terms from each other yields
\begin{align}
	\nonumber \kappa_p\xi_m-\xi_s = \ &\big(\mathcal{V}_s - \kappa_p\mathcal{V}_m\big) + \Lambda\big(\mathcal{P}_s - \kappa_p\mathcal{P}_m\big) \\
	+ &\big(\tilde{\mathcal{V}}_s - \kappa_p\tilde{\mathcal{V}}_m\big) + \Lambda\big(\tilde{\mathcal{P}}_s - \kappa_p\tilde{\mathcal{P}}_m\big).
	\label{eq:teleopTrackingError}
\end{align}
It can be easily seen that \eqref{eq:teleopTrackingError} can be further written as
\begin{equation}
	\kappa_p\xi_m-\xi_s = \tilde{\mathcal{Z}} + \mathcal{Z}
\end{equation}
where
\begin{equation}
	\label{eq:TeleopError}
	\mathcal{Z} \equiv \big(\mathcal{V}_s - \kappa_p\mathcal{V}_m\big) + \Lambda\big(\mathcal{P}_s - \kappa_p\mathcal{P}_m\big)
\end{equation}
and $ \tilde{\mathcal{Z}} $ is obtained from $ {\mathcal{Z}} $ using \eqref{eq:filter}. 

Following Lemma 2.4 in \cite{Zhu2000} yields
\begin{equation}
	\label{eq:convergenceTeleopAsWhole}
	\mathcal{Z} \in L_2 \bigcap L_\infty.
\end{equation}
Eventually, it follows from \eqref{eq:TeleopError}, \eqref{eq:convergenceTeleopAsWhole} and Lemma 1 in~\cite{Zhu2000}~that
\begin{align}
	\label{eq:xiV}
	\xi_v = \kappa_p\mathcal{V}_m - \mathcal{V}_s \in L_2 \bigcap L_\infty\\
	\label{eq:xiP}
	\xi_p = \kappa_p\mathcal{P}_m - \mathcal{P}_s \in L_2 \bigcap L_\infty
\end{align}
hold, which guarantees the $ L_2 $ and $ L_\infty $ stability of the velocity and position tracking of the teleoperation system.

\subsection{Transparency}

Transparency of the teleoperation system can be analyzed by substituting \eqref{eq:teleoperationVmd} and \eqref{eq:teleoperationVsd} into \eqref{eq:convergenceMaster} and \eqref{eq:convergenceSlave}, then summing the resulting error terms together results in
\begin{align}
	\nonumber
	\kappa_p\xi_m+\xi_s = \ &\big(\tilde{\mathcal{V}}_s - {\mathcal{V}}_s\big) + \kappa_p\big(\tilde{\mathcal{V}}_m - {\mathcal{V}}_m\big) + \Lambda\big(\tilde{\mathcal{P}}_s - {\mathcal{P}}_s\big) \\ 
	&+ \Lambda\kappa_p\big(\tilde{\mathcal{P}}_m - {\mathcal{P}}_m\big) - 2\mathbf{ A }\big(\tilde{\mathbf{f}}_s + \kappa_f{\tilde{\mathbf{f}}}_m\big).
	\label{eq:transparencyErrorv1}
\end{align}
Substituting \eqref{eq:filter} with $ -\mathbf{C}^{-1}\dot{\tilde{\mathcal{X}}} = \left(\tilde{\mathcal{X}} - \mathcal{X}\right) $ into \eqref{eq:transparencyErrorv1} yields
\begin{align}
	\kappa_p\xi_m+\xi_s = \ & -\mathbf{ C }^{-1}\left[ \dot{\tilde{\mathcal{V}}}_s + \kappa_p\dot{\tilde{\mathcal{V}}}_m + \Lambda{\tilde{\mathcal{V}}}_s + \Lambda\kappa_p{\tilde{\mathcal{V}}}_m  \right] - 2\mathbf{ A }\big(\tilde{\mathbf{f}}_s + \kappa_f{\tilde{\mathbf{f}}}_m\big).
	\label{eq:teleopTransparencyError}
\end{align}
According to \cite{Zhu2000}, we can rewrite \eqref{eq:teleopTransparencyError}, using \eqref{eq:xiV}--\eqref{eq:xiP}, as
\begin{equation}
	- \tilde { \mathbf { f } } _ { m } = \kappa _ { f } ^ { - 1 } \tilde { \mathbf { f } } _ { s } + \kappa _ { f } ^ { - 1 } \kappa _ { p } \mathbf { A } ^ { - 1 } {} \mathbf { C } ^ { - 1 } ( s + \Lambda ) \tilde { \mathcal { V } } _ { m } + \frac { \xi } { 2 \kappa _ { f } }
	\label{eq:transparency}
\end{equation}
where $ s $ denotes the Laplace operator, and
\begin{equation}
	\xi \equiv \mathbf{ A } ^ { - 1 } [ - \mathbf{ C } ^ { - 1 } ( s + \Lambda ) \tilde { \xi } _ { v } + ( \xi _ { s } + \kappa _ { p } \xi _ { m } ) ]
\end{equation}
where $ \tilde{ \xi }_v $ is obtained from $ \xi_{ v } $ using \eqref{eq:filter}. In view of Lemma~1 in \cite{Zhu2000}, the following holds true
\begin{equation}
	\xi \in L_2 \bigcap L_\infty.
\end{equation}

Transparency of the teleoperation system can be clearly~seen from \eqref{eq:transparency}. Within a limited frequency range, the filtered signals can be assumed to be approximately equal to their non-filtered counterparts. The last term on the right hand side of \eqref{eq:transparency} is bounded to converge to zero. Then, transparency error can be described by the second term on the right hand side of \eqref{eq:transparency}. It comprises of velocity and acceleration dependent terms. The acceleration related term $  \kappa_f^{-1}{}{\kappa_p}\mathbf{ A }^{-1}\mathbf{ C }^{-1} $ acts as a \textit{virtual mass} on the teleoperation system, while the velocity dependent term $  \kappa_f^{-1}{}{\kappa_p}\mathbf{ A }^{-1}\mathbf{ C }^{-1}\Lambda $ determines the \textit{damping} of the teleoperation system.

\subsection{Stability under time delay}
\label{sec:time delay}

time delay under teleoperation is a much investigated issue especially in space teleoperation. Although the focus of this study is in terrestrial applications, robustness against arbitrary time delay of the proposed method is discussed briefly. In \cite{ZhaiTeleop2018} similar approach was used for longer and varying delays.

Without loss of generality, we consider a one-dimensional system in the stability analysis. The extension to multiple-dimensional systems can be proceeded accordingly. Due to the fact that both master and slave manipulators~have independent stability-guaranteed controllers linked only by the communication channel, the stability under time delay can be analyzed by modifying \eqref{eq:teleoperationVmd} and \eqref{eq:teleoperationVsd}~as
\begin{align}
	\mathcal{V}_{md} &= 
	\kappa_{p}^{-1} \left[ e^{-sT} \left( \tilde{\mathcal{V}}_s + \Lambda \tilde{\mathcal{P}}_s \right) - \kappa_p\Lambda\mathcal{P}_m \right.
	\label{eq:Vmd delay}
- \left. \mathbf{A}\left( e^{-sT} \tilde{\mathbf{f}}_s + \left({\kappa}_f - {\kappa}_p\right) \tilde{\mathbf{f}}_m\right) \right]
	\\
	\label{eq:Vsd delay}
	\mathcal{V}_{sd} &= 
	e^{-sT}\kappa_p \left( \tilde{\mathcal{V}}_m - \Lambda\tilde{\mathcal{P}}_m \right) - \Lambda\mathcal{P}_s
-  e^{-sT}\mathbf{A} {\kappa}_f\tilde{\mathbf{f}}_m
\end{align}
where the communication channel is represented as pure time delay of $ T $; see Fig. \ref{fig:DelayDiagram}. The stability under arbitrary time delay can be analyzed similarly to the method presented in \cite{Zhu2000}. 

Fig.~\ref{fig:DelayDiagram} represents the teleoperation system under arbitrary time delay based on \eqref{eq:Vmd delay} and \eqref{eq:Vsd delay}. In the figure, $Z_h$ is the operator dynamics defined in \eqref{eq:LTImodel} (now considered one-dimensional) and $ Z_e $ is environment dynamics; here approximated with second-order linear dynamics.

To analyze the effect of time delay on the system stability, transfer functions for both sides of the communication channel need to be defined. The transfer function for master side, from input D to output A (see Fig.~\ref{fig:DelayDiagram}) can be formed as
\begin{align}
	G_m 
	= & \frac{\frac{C}{s+C} - \frac{sAC\frac{\kappa_{f}}{\kappa_{p}} Z_h}{\left(s+\Lambda\right)\left(s+C\right) }}{1 + \frac{sAC\frac{\kappa_{f}}{\kappa_{p}} Z_h}{\left(s+\Lambda\right)\left(s+C\right) }}\\
	= & \frac{C\left(s+\Lambda\right)-sAC\frac{\kappa_{f}}{\kappa_{p}}Z_h}{\left(s+\Lambda\right)\left(s+C\right) + sAC\frac{\kappa_{f}}{\kappa_{p}}Z_h}\\
	= & \tfrac{-AC M_h^* s^2 + \left(C - AC D_h^*\right)s + \left(C\Lambda- AC K_h^*\right)}{\left(1 + AC M_h^*\right)s^2 + \left(\Lambda + C + AC D_h^*\right)s + \left(\Lambda C + AC K_h^*\right)}
\end{align} 
where $M_h^* = \frac{\kappa_{ f }}{\kappa_{ p }}{M}_h$, $D_h^* = \frac{\kappa_{ f }}{\kappa_{ p }}{D}_h$ and $K_h^* = \frac{\kappa_{ f }}{\kappa_{ p }}{K}_h$.

Following the same procedure, transfer function from the input B of the slave side to output C can be formed as
\begin{align}
	G_s 
	= & \frac{\frac{C}{s+C} - \frac{sACZ_e}{\left(s+\Lambda\right)\left(s+C\right) }}{1 + \frac{sACZ_e}{\left(s+\Lambda\right)\left(s+C\right)}}\\
	= & \frac{C\left(s+\Lambda\right)-sACZ_e}{\left(s+\Lambda\right)\left(s+C\right) + sACZ_e}.
\end{align}
Assume flexible environment with dynamics as
\begin{equation}
{Z}_e = {M}_e s + {D}_e + \frac{{K}_e}{s}
\end{equation}
where $ {M}_e $, ${D}_e $ and ${K}_e $ define the inertia, damping and stiffness of the environment, respectively.

Then, the transfer function from B to C can be written as
\begin{equation}
	G_s = \tfrac{-AC M_e s^2 + \left(C - AC D_e\right)s + \left(C\Lambda- AC K_e\right)}{\left(1 + AC M_e\right)s^2 + \left(\Lambda + C + AC D_e\right)s + \left(\Lambda C + AC K_e\right)}.
\end{equation}

To guarantee stability under arbitrary time delay, the gain~of each manipulator together with their respective local controllers must remain equal or smaller than one across~the~entire frequency spectrum. Thus, to ensure stability of the entire teleoperation system with arbitrary time delay (see Fig. \ref{fig:DelayDiagram}), the following conditions need to be satisfied
\begin{align}
	\label{eq:cond1}
	&\left\lVert
	\tfrac
	{-AC M_h^*(j\omega)^2 + \iso( {C} - AC D_h^* \iso)(j\omega) + \iso( {C}\Lambda - AC K_h^* \iso)}
	{ \iso(1 + AC M_h^*\iso)(j\omega)^2 + \iso( \Lambda + {C} + AC D_h^* \iso) (j\omega) + \iso( \Lambda {C} + AC K_h^* \iso)}
	\right\rVert_\infty \hspace{-0.05cm} \le \hspace{-0.05cm} 1
	\\
	\label{eq:cond2}
	&\left\lVert
	\tfrac
	{-AC M_e \, (j\omega)^2 + \iso( {C} - AC D_e \, \iso)(j\omega) + \iso( {C}\Lambda - AC K_e \, \iso)}
	{ 	  \iso( 1 + AC M_{e} \iso)(j\omega)^2 
		+ \iso( \Lambda + {C} + AC D_e \iso)(j\omega) 
		+ \iso( \Lambda {C}   + AC K_e \iso) }
	\right\rVert_\infty \le 1.
\end{align}

To satisfy the stability conditions in \eqref{eq:cond1} and \eqref{eq:cond2}, the following relation must be satisfied
\begin{align}
&\left[ \left( \Lambda {C} + {AC}K \right) - \left( 1 + {AC}M \right) \omega^2 \right]^2 + \left[ \left(  \Lambda + {C} + {AC}D \right) \omega \right]^2 \nonumber \\ 
& - \left[ \left( {C} - {AC}D\right) \omega \right]^2 - \left[ \left( \Lambda {C} - {AC}K \right) + {AC}M \omega^2 \right]^2 \ge 0 
\label{APP_EQ1}
\end{align}
for both the slave and master manipulators by substituting $M$, $D$ and $K$ with $M_e$, $D _e$ and $K_e$ (the slave side), or $M_h^*$, $D_h^*$ and $K_h^*$ (the master side), respectively. Furthermore, to satisfy \eqref{APP_EQ1},
\begin{equation}
\mathbf{a}\omega^4 + \mathbf{b}\omega^2 + \mathbf{c} \ge 0 \label{APP_EQ2}
\end{equation}
must hold, where
\begin{equation} \label{APP_EQ3}
\begin{cases} 
\mathbf{a} = 1 + 2 {AC}M	\\
\mathbf{b} = \Lambda^2 + 2AC \left( 2ACD - K - 2 \Lambda {AC} M \right)\\
\mathbf{c} = 4 \Lambda {AC}^2 K.
\end{cases}
\end{equation}

It follows directly from the positive-definite properties of $M$, $D$, $K$, $C$ and $\Lambda$ that $\mathbf{a} \ge 0$ and $\mathbf{c} > 0$ hold indefinitely.~Consequently, it follows from \eqref{APP_EQ1}--\eqref{APP_EQ3} that   
\begin{equation}
\mathbf{b} + 4 C \sqrt{\Lambda A K \left( 2 AC M + 1 \right)} \ge 0
\end{equation}
must hold to fulfill the stability conditions in \eqref{eq:cond1} and \eqref{eq:cond2}.

\begin{figure}[t]
	\centering
	\includegraphics[width=0.75\columnwidth]{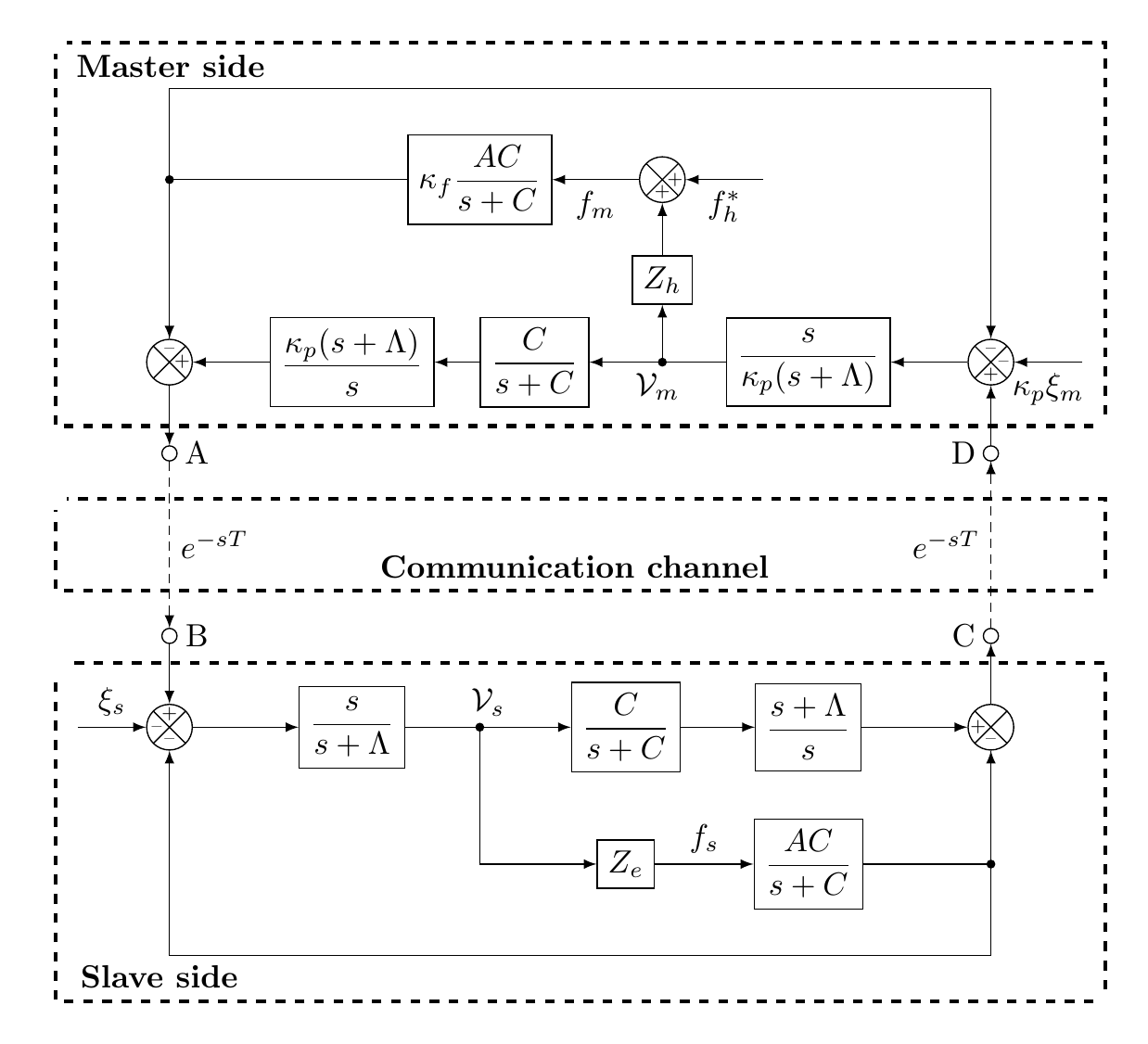}
	\caption{ One-dimensional block diagram representation of the teleoperation system under arbitrary time delay of $ T $. }
	\label{fig:DelayDiagram}
\end{figure}	

\section{Experiments}
\label{sec:Experiments}
This section evaluates the performance of the proposed teleoperation system. First, Section \ref{sec:implementation} addresses the system implementation issues. Then, Section \ref{sec:no_delay} provides the experiments without time delay, followed by the experiment with time delay in Section \ref{sec:delay}.

\subsection{Experiment description}
\label{sec:implementation}
The experimental~implementation comprises four main components, visualised in Fig.~\ref{fig:ExpImpl}, which are the electric master manipulator (Phantom Premium 6DOF/3.0L haptic device), the host computer for the master manipulator, the real-time computer and the hydraulic slave manipulator (HIAB-031 manipulator). The two-DOF hydraulic manipulator (in Fig. \ref{fig:setup}) has a maximum reach of approximately 3.2~m, and a payload of 475~kg is attached to its tip. For the real-time control system, the following components were used: a DS1005 processor board, a DS3001 incremental encoder board, a DS2103 DAC board, a DS2003 ADC board, and a DS4504 100 Mb/s ethernet interface. The remaining hardware implementations can be found in \cite{Koivumaki_TRO2015}, \cite{KoivumakiTMECH2017} or \cite{LampinenCASE}. Control computations have been run with 500 Hz frequency. The communication channel parameters of each experiment are shown in Table \ref{table:parameters}.

\begin{table}[t]
	\centering
	\caption{Used communication channel parameters.}
	\label{table:parameters}
	\begin{tabular}{l c c c c c}
		{} & {$\kappa_{ p } $} & {$ \kappa_{ f } $} & {$ \Lambda $} & {$ \mathbf{A} $} & {$ \mathbf{C} $}  \\[0.2em]
		\hline \\[-0.9em]
		{Experiment 1:}& {1} & {300} & {2.0} & {60$\times 10^{-6}$} & {35}  \\
		{Experiment 2:}& {4} & {800} & {2.0} & {100$\times 10^{-6}$} & {35}   \\
		{Experiment 3:}& {1.5} & {500} & {1.5} & {40$\times 10^{-6}$} & {35}  
	\end{tabular}
\end{table}

\begin{figure}[b]
	\begin{center}
		\includegraphics[width=0.6\columnwidth]{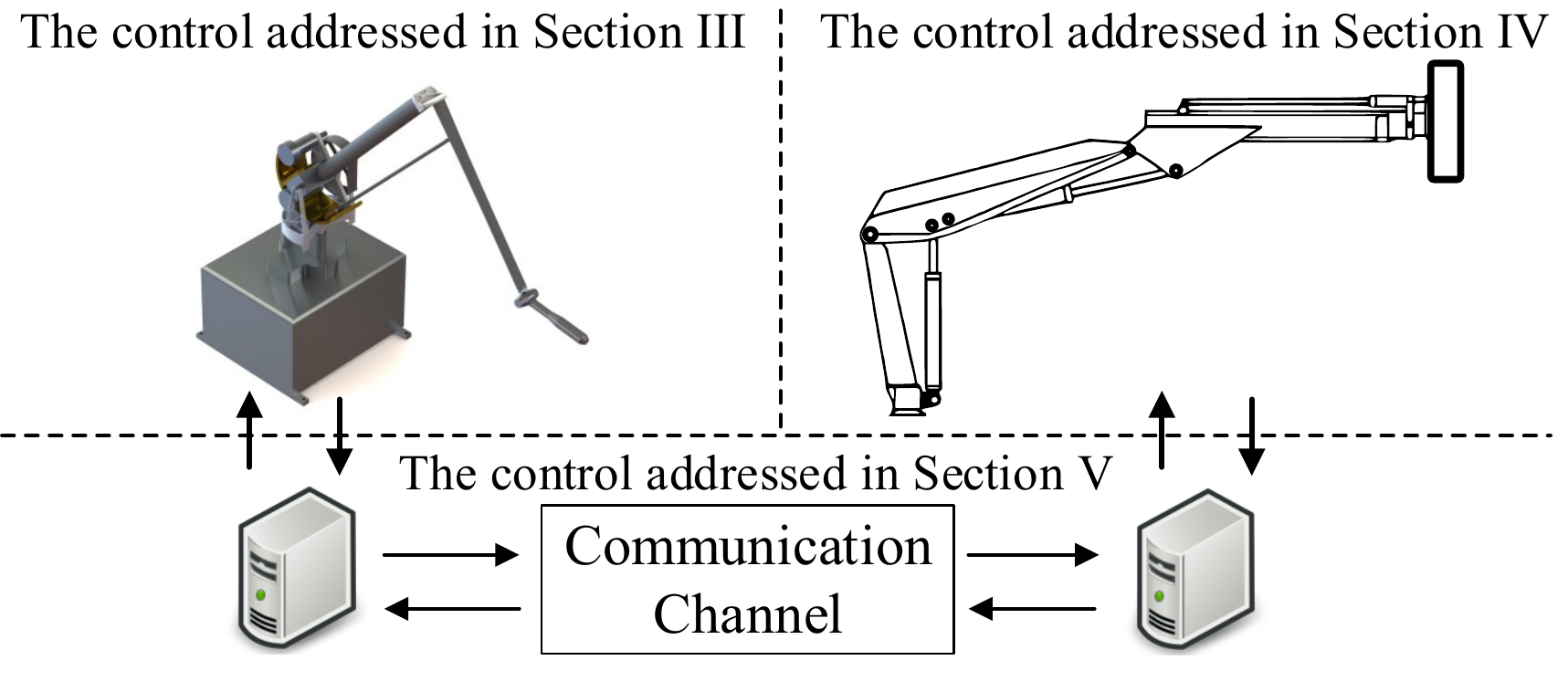}
	\end{center}
	\caption{High-level overview of the experimental implementation.}
	\label{fig:ExpImpl}
\end{figure}

\begin{figure}[t]
	\begin{center}
		\includegraphics[width=0.75\columnwidth]{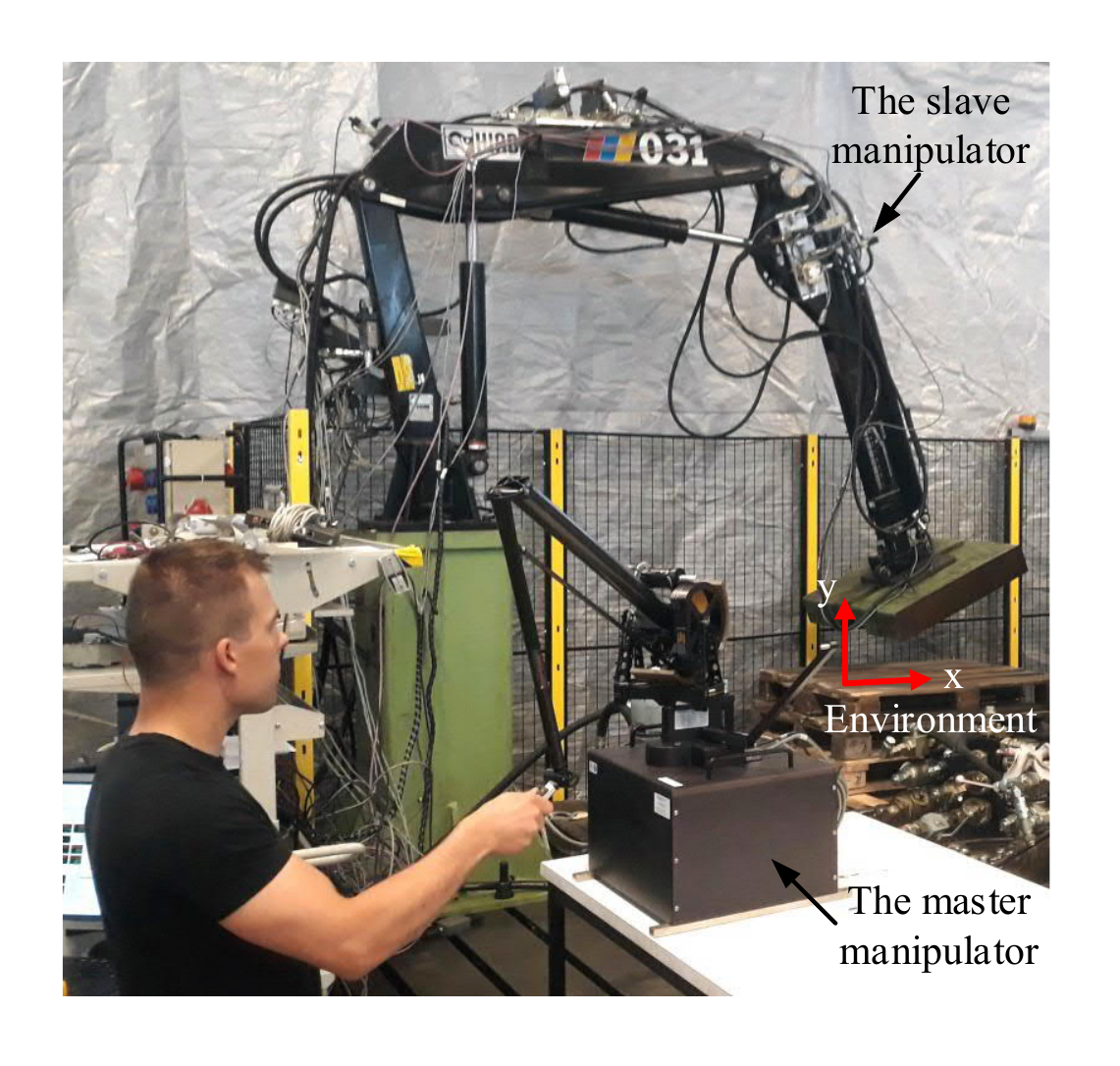}
	\end{center}
	\caption{Experimental implementation and setup.}
	\label{fig:setup}
\end{figure}

\begin{figure*}[t]
	\centering
	\includegraphics[width=1\textwidth]{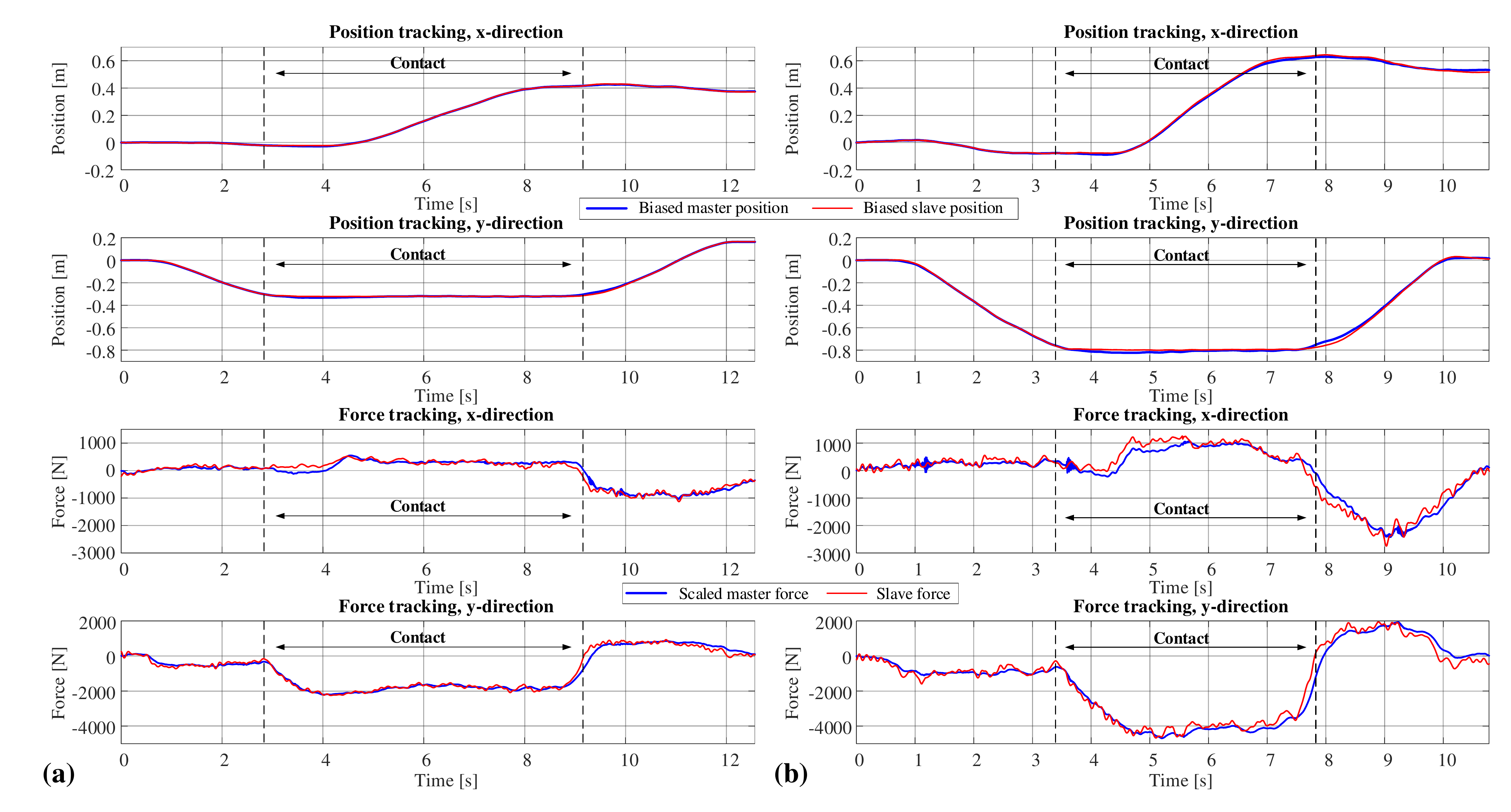}
	\caption{Results from experimental implementation with bilateral force-reflected teleoperation. In the experiments, two different set of scaling parameters in the communication system were used as: a) $ \kappa_{p} = 1.0 $ and $ \kappa_{ f } = 300 $, and b) $ \kappa_{p} = 4.0 $ and $ \kappa_{ f } = 800 $. Scaling of the axes of the figures is kept the same between the two experiments to elaborate the effect of the scaling parameters to the behavior of the system.}
	\label{fig:tracking}
\end{figure*}

\begin{figure}[h!]
	\centering
	\includegraphics[width=0.8\columnwidth]{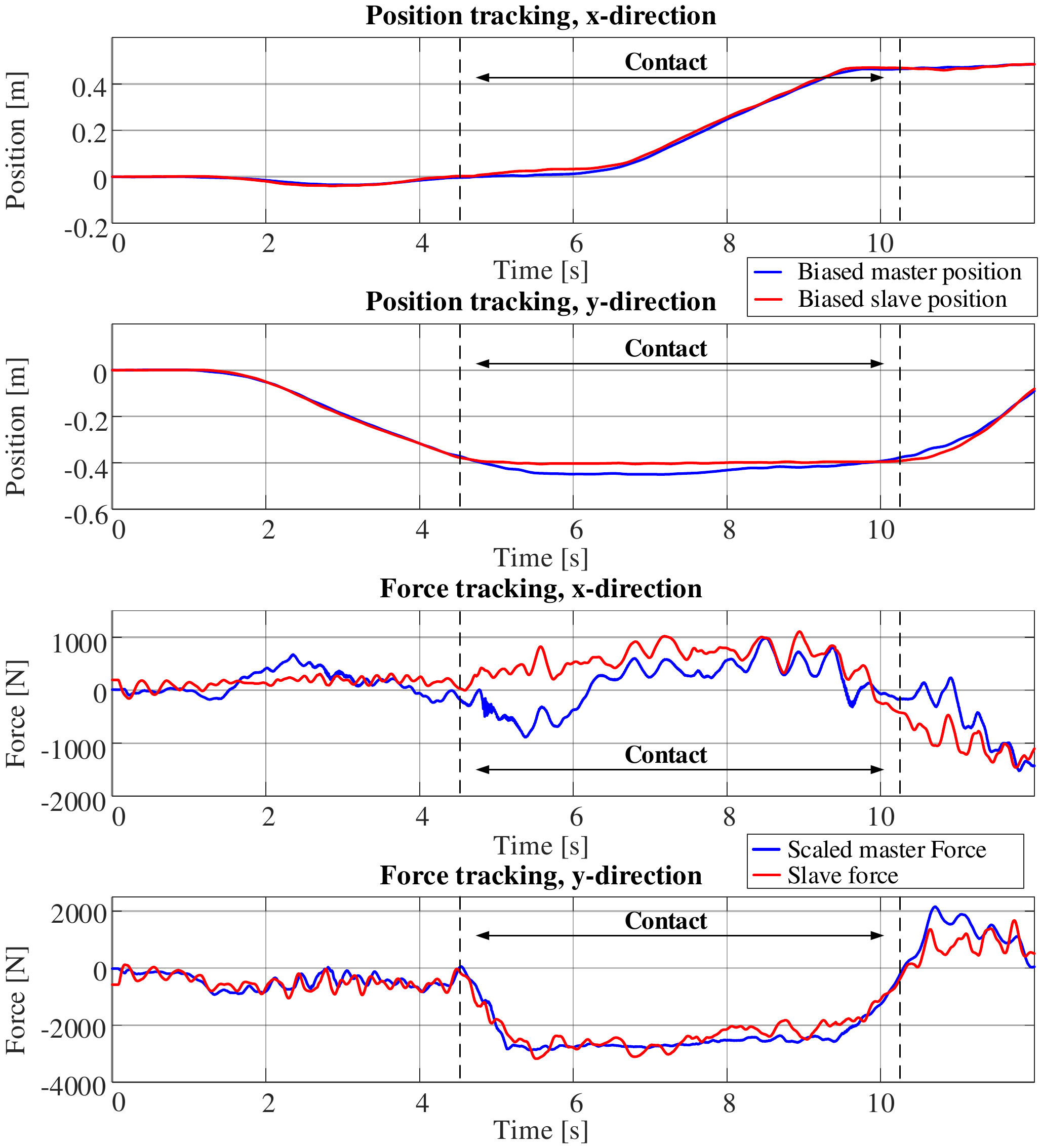}
	\caption{ Teleoperation task under one-way time delay of 80ms. Parameters of the communication channel were set as: $ \kappa_{p} = 1.5 $ and $ \kappa_{ f } = 500 $, $ \Lambda = 1.5 $, $ \mathbf{C} = 35 $ and $ \mathbf{A} = 40\times 10^{-6}$. }
	\label{fig:delay}
\end{figure}

Teleoperation control between the master and slave manipulators was engaged by pressing a pushbutton and disengaged~by releasing the button. At first, the slave manipulator was driven from free space to contact with the environment along the Cartesian y-axis; see Fig.~\ref{fig:setup} for the directions of the Cartesian coordinate system. After contact with the environment was established, the slave manipulator was driven along the surface of the pallets, while maintaining constant force against the environment. Finally, the slave manipulator was driven back to free space after approximately 0.5~m of sliding against the wooden pallets. The above described task was repeated \textit{without any time delay} with two different sets of motion/force scaling parameters ($\kappa_{ p }$ and $\kappa_{ f }$) of the communication system; see Fig. \ref{fig:tracking}. Then, an experiment \textit{with 80~ms one-way time delay} in the communication channel was performed; see Fig. \ref{fig:delay}. In Figs. \ref{fig:tracking} and \ref{fig:delay}, the master manipulator data is shown in blue, and the slave manipulator data in red. 

\subsection{The experiments without time delay}
\label{sec:no_delay}

In the \textit{first experiment} (Fig. \ref{fig:tracking}a), 1:1 position mapping was used between the manipulators ($\kappa_{ p } = 1$), while forces of the master manipulator were scaled up by a factor of $\kappa_{ f } = 300$; see Table \ref{table:parameters}. In the \textit{second experiment} (Fig. \ref{fig:tracking}b), 4:1 position scaling between the master and slave manipulator was used ($ \kappa_{ p } = 4 $), yielding 4 times larger movement of the slave manipulator compared to the movement of the master manipulator, along with force scaling by a factor of $ \kappa_{ f } = 800 $; see Table \ref{table:parameters}. 

As Fig. \ref{fig:tracking} shows, accurate position (see the \textit{first}~and~\textit{second rows}) and force tracking (the \textit{third} and \textit{fourth rows}) between the master and slave manipulators is achieved with different scaling parameters, as predicted by the theory, despite~the~inherent challenges of force control of hydraulic manipulator in the teleoperation system (see the discussion in Section \ref{sec:intro}). The forces in third and fourth rows presents the \textit{estimated} contact forces of the slave- and master manipulator along the x- and y-axes, respectively. In the results, the master manipulator forces are scaled up by the respective scaling factor. 
Note that the slave manipulator contact forces are estimated from the cylinders' chamber pressures. Thus, some inaccuracies can exist in the measured contact forces (see \cite{KoivumakiTMECH2017} for more details). However, the operator is still able to effectively sense the contact forces between the slave and the environment, and excessive contact forces can be prevented.~It~is~valid~to mention that the proposed force-sensorless approach provides a practical solution for teleoperation of extremely powerful hydraulic manipulators, as conventional six-DOF force/torque sensors are fragile and prone to overloading \cite{Koivumaki_TRO2015}.

\begin{remark}
	Note that in the second experiment the transition from free space motion to contact motion, was done rapidly with a velocity of approximately 0.2~m/s. This is to demonstrate the stable behavior of the control system even with high velocity and rapid changes of system states.
\end{remark}

\subsection{The experiment with one-way time delay of 80 ms}
\label{sec:delay}

Fig.~\ref{fig:delay} shows respective results as in Fig. \ref{fig:tracking} with one-way time delay of 80 ms. To demonstrate the versatility, in the experiment the scaling factors were selected as $\kappa_{p} = 1.5$ and $\kappa_{f} = 500$. As the results indicate, the proposed method (\textit{i}) is robust against time delays and communication noises (as predicted~by~the theory in Section \ref{sec:time delay}) and (\textit{ii})~is capable of handling the delay without significant loss~of~performance.

\section{Conclusions}
\label{sec:Conclusions}
This study presented force-reflected bilateral teleoperation using asymmetric electrical master and hydraulic slave manipulators. The control design takes advantage of the VDC approach, which allows us to design local subsystem-dynamics-based controllers for the master and slave manipulators independently. Then, the teleoperation system was completed by designing the communication channel between the master and slave controllers as motivated by \cite{Zhu2000}. The teleoperation scheme provided unique features, such as arbitrary motion/force scaling between the manipulators, effectively enabling the connection of two very dissimilar manipulators.

The experimental results demonstrated the performance of the proposed method and showed excellent motion and force tracking between the manipulators. Furthermore,~robustness against an arbitrary time delay was demonstrated theoretically and experimentally. Similar to our previous studies \cite{Koivumaki_TRO2015,KoivumakiTMECH2017}, and \cite{Koivumaki_ASME2017}, tracking performance improvements can be expected after rigorous application of a full parameter adaptation implementation and with tuning of the system parameters.

This paper advances force-reflected bilateral teleoperation control one more step toward practical applications and implements novel features to make it applicable to a large class of manipulators remotely operated over 5G cellular network. The theoretical and experimental studies in this paper concluded that the use of a force-sensor-less design for bilateral teleoperation is feasible. 

Future work will focus on maximizing the system performances of position tracking and transparency, as well as forming basis functions for human operator exogenous forces through machine learning. Moreover, we intend to focus on expanding the experimental system to possess more degrees of freedom with the goal of achieving 6-DOF manipulation.

\appendix
\section{$L_2$ and $L_{\infty}$ Stability}
\label{App:Lebesgue_Stab}

Definition~\ref{def:Lspace} provides a definition for the Lebesgue space.  

\begin{definition}[\hspace{-0.10cm} \cite{Zhu2010Virtual}] \label{def:Lspace}
The Lebesgue space, denoted as $L_p$ with $p$ being a positive integer, contains all Lebesgue measurable and integrable functions $f(t)$ subject to
\begin{equation}
		\|f\|_p = \lim_{T\to\infty}\left[\int\limits_0^T|f(t)|^pd\tau\right]^\frac{1}{p}<+\infty.
\end{equation}
Two particular cases are considered:
\begin{enumerate}
  \item[(a)] A Lebesgue measurable function $f(t)$ belongs to $L_2$ if and only if \\ $\ \lim_{T\to\infty}\int_0^T|f(t)|^2d\tau < +\infty$.
  \item[(b)] A Lebesgue measurable function $f(t)$ belongs to $L_\infty$ if and only if \\$\ {\rm max}_{t\in[0,\infty)}|f(t)| < +\infty$.
\end{enumerate}
\end{definition}

Lemma~\ref{lem:stab} (a simplified version of Lemma 2.3 in~\cite{Zhu2010Virtual})~provides that a system is $L_2$ and $L_\infty$ stable with its affiliated~vector $\mathbf{x}(t)$, being a function in $L_\infty$ and its affiliated vector~$\mathbf{y}(t)$, being a function in $L_2$. 
\begin{lemma}[\hspace{-0.1cm} \cite{Zhu2010Virtual}] \label{lem:stab}
Consider a non-negative differentiable~function $\xi(t)$ defined as
\begin{equation}
		\xi(t) \geqslant \frac{1}{2}\mathbf{x}(t)^T\mathbf{P}\mathbf{x}(t) 
\end{equation}
with $\mathbf{x}(t) \in \mathbb{R}^n$, $n \geqslant 1$, and $\mathbf{P} \in \mathbb{R}^{n\times n}$ being a symmetric positive-definite matrix. If the time derivative of $\xi(t)$ is Lebesgue integrable and governed by
\begin{equation}
		\dot{\xi}(t) \leqslant - \mathbf{y}(t)^T\mathbf{Q}\mathbf{y}(t) - s(t) 
\end{equation}
where $\mathbf{y}(t) \in \mathbb{R}^m$, $m \geqslant 1$, and $\mathbf{Q} \in \mathbb{R}^{m\times m}$ being a symmetric positive-definite matrix, and $s(t)$ is subject to
\begin{equation}
		 \int_0^\infty s(t)dt \geqslant - \gamma_0 \label{EQ_s(t)} 
\end{equation}
with $0 \leqslant \gamma_0 < \infty$, then, it follows that $\xi(t) \in L_{\infty}$, $\mathbf{x}(t) \in L_{\infty}$ and $\mathbf{y}(t) \in L_{2}$ hold.
\end{lemma}

Lemma~\ref{lem:Stab_asymp} provides an alternative to Barbalat's lemma.
\begin{lemma}[\hspace{-0.2cm} \cite{Tao1997}] \label{lem:Stab_asymp}
 If $e(t) \in L_2$ and $\dot{e}(t) \in L_\infty$, then $\displaystyle{\lim_{t \to \infty} e(t) = 0}$.
\end{lemma}
\begin{remark} \label{remark:lebsque_stab}
As a distinction to Lyapunov approaches, Lemma~\ref{lem:stab} allows different appearances of variables in the non-negative function itself and in its time-derivative. When all error signals are proven to belong to $L_2$ and $L_\infty$ in the sense of Lemma~\ref{lem:stab}, then asymptotic stability can be proven with Lemma~\ref{lem:Stab_asymp}, if the time-derivatives of all error signals belong~to~$L_\infty$. Note that $s(t) = 0$ is a special case that satisfies \eqref{EQ_s(t)} in Lemma \ref{lem:stab}.
\end{remark}

\section{Virtual Stability}
\label{App:Virtual_Stab}

The unique feature of the VDC approach is the introduction of a scalar term, namely the \textit{virtual power flow} (VPF)~\cite{Zhu2010Virtual}; see Definition~\ref{def:VPF}. The VPFs uniquely define the dynamic interactions among the subsystems and play an important role in the definition of \textit{virtual stability}~\cite{Zhu2010Virtual}, which is defined in a simplified form in Definition~\ref{def:Vstab}.
\begin{definition}[\hspace{-0.15cm} \cite{Zhu2010Virtual}] \label{def:VPF}
The \textit{virtual power flow} with respect to frame $\{{\mathbf A}\}$ is the inner product of the linear/angular velocity vector error and the force/moment vector error as
\begin{equation}
p_{\mathbf{A}} = ({}^{{\mathbf A}}V_{\rm r}-{}^{{\mathbf A}}V)^{T}({}^{{\mathbf A}}F_{\rm r}-{}^{{\mathbf A}}F) \label{EQ_VPF}
\end{equation}
where ${}^{{\mathbf A}}V_{\rm r} \in \mathbb{R}^{6}$ and ${}^{{\mathbf A}}F_{\rm r} \in \mathbb{R}^{6}$ represent the required vectors of ${}^{{\mathbf A}}V \in \mathbb{R}^{6}$ and ${}^{{\mathbf A}}F \in \mathbb{R}^{6}$, respectively.
\end{definition}

\begin{definition}[\hspace{-0.10cm} \cite{Zhu2010Virtual}] \label{def:Vstab}
A subsystem with a \textit{driven} VCP to which frame $\{\mathbf A\}$ is attached and a \textit{driving} VCP to which frame $\{\mathbf C\}$ is attached is said to be virtually stable with its affiliated vector $\mathbf{x}(t)$ being a virtual function in $L_{\infty}$ and its affiliated vector $\mathbf{y}(t)$ being a virtual function in $L_{2}$, if and only if there exists a non-negative accompanying function
\begin{equation}
{\nu (t)} \geqslant \frac{1}{2} \mathbf{x}(t)^{T}\mathbf{P}\mathbf{x}(t)
\label{EQ8}
\end{equation}
such that
\begin{equation}
{\dot\nu (t)} \leqslant -\mathbf{y}(t)^{T}\mathbf{Q}\mathbf{y}(t) - s(t) + p_{\mathbf{A}} - p_{\mathbf{C}}
\label{EQ9}
\end{equation}
holds, 
\begin{equation}
\int_0^\infty s(t)d\tau \geqslant -\gamma_s \label{EQ10}
\end{equation}
where $0 \leqslant \gamma_0 < \infty$, $\mathbf{P}$ and $\mathbf{Q}$ are two block-diagonal positive-definite matrices, and $p_{\mathbf{A}}$ and $p_{\mathbf{C}}$ denote the \textit{virtual power flows} (by Definition~\ref{def:VPF}) at frames $\{\mathbf A\}$ and $\{\mathbf C\}$, respectively.
\end{definition}
\begin{remark} \label{remark:overall_stab}
In view of Theorem 2.1 in~\cite{Zhu2010Virtual}, when all subsystems qualify as \textit{virtually stable} (in the sense of Definition~\ref{def:Vstab}), the $L_2$ and $L_\infty$ stability of the entire system can be guaranteed in the sense Lemma~\ref{lem:stab}.
\end{remark}

\section{The Proof of Theorem \ref{thm:O2}}
\label{proof:thm_O2}

According to the Definition \ref{def:VPF} and \eqref{eq:fs}, \eqref{eq:CV}, \eqref{EQ_GF}, \eqref{eq:fsr}, \eqref{eq:CVr} and \eqref{EQ_GFr} if follows~that
\begin{align}
	p_{\mathbf G} &= ({}^{{\mathbf G}}V_{\rm r} - {}^{{\mathbf G}}V)^T({}^{{\mathbf G}}F_{\rm r} - {}^{{\mathbf G}}F) \nonumber \\
	& = 
	({}^{{\mathbf G}}V_{\rm r} - {}^{{\mathbf G}}V)^T {\rm diag}({}^{\mathbf G}{\mathbf R}_{\mathbf C},{}^{\mathbf G}{\mathbf R}_{\mathbf C}) \nonumber \\
	& \phantom{ = } \times 
	\left[
		\begin{smallmatrix}
			\mathbf{ I }_{2\times2} & \boldsymbol{0}_{2\times4}
		\end{smallmatrix}
	\right]^T
	\sigma_f \left( \mathbf{f}_{sr} - \mathbf{f}_{s} \right) \nonumber\\
	& =
	\big({}^{{\mathbf G}}V_{\rm r} - {}^{{\mathbf G}}V\big)^T {\rm diag}({}^{\mathbf G}{\mathbf R}_{\mathbf C},{}^{\mathbf G}{\mathbf R}_{\mathbf C}) \nonumber \\
	& \phantom{ = } \times 
	\left[
		\begin{smallmatrix}
			\mathbf{ I }_{2\times2} & \boldsymbol{0}_{2\times4}
		\end{smallmatrix}
	\right]^T
	\sigma_f \mathbf{M}_e \big( \dot{\mathcal{V}}_{sr} - \dot{\mathcal{V}}_{s} \big) \nonumber\\
	& =
	\big({}^{{\mathbf C}}V_{\rm r} - {}^{{\mathbf C}}V\big)^T 
	\left[
	\begin{smallmatrix}
	\mathbf{ I }_{2\times2} & \boldsymbol{0}_{2\times4}
	\end{smallmatrix}
	\right]^T
	\sigma_f \mathbf{M}_e \big( \dot{\mathcal{V}}_{sr} - \dot{\mathcal{V}}_{s} \big) \nonumber \\
	& =
	\sigma_f \big( {\mathcal{V}}_{sr} - {\mathcal{V}}_{s} \big)^T 
	\mathbf{M}_e \big( \dot{\mathcal{V}}_{sr} - \dot{\mathcal{V}}_{s} \big).
\end{align}
For constant value of $ \sigma_f $ the following holds true
\begin{align}
	\label{eq:slaveVirtualStbaility}
	\int_{0}^{\infty} p_{\mathbf G} dt
	&=\int_{0}^{\infty} \sigma_f \big( {\mathcal{V}}_{sr} - {\mathcal{V}}_{s} \big)^T 
	\mathbf{M}_e \big( \dot{\mathcal{V}}_{sr} - \dot{\mathcal{V}}_{s} \big) dt \\
	&\ge - \frac{1}{2} \sigma_f \big( {\mathcal{V}}_{sr}(0) - {\mathcal{V}}_{s}(0) \big)^T 
	\mathbf{M}_e \big( {\mathcal{V}}_{sr}(0) - {\mathcal{V}}_{s}(0) \big) \nonumber
\end{align}

\bibliographystyle{IEEEtran}
\bibliography{bibliography}

\begin{thebibliography}{10}
\providecommand{\url}[1]{#1}
\csname url@samestyle\endcsname
\providecommand{\newblock}{\relax}
\providecommand{\bibinfo}[2]{#2}
\providecommand{\BIBentrySTDinterwordspacing}{\spaceskip=0pt\relax}
\providecommand{\BIBentryALTinterwordstretchfactor}{4}
\providecommand{\BIBentryALTinterwordspacing}{\spaceskip=\fontdimen2\font plus
\BIBentryALTinterwordstretchfactor\fontdimen3\font minus
  \fontdimen4\font\relax}
\providecommand{\BIBforeignlanguage}[2]{{%
\expandafter\ifx\csname l@#1\endcsname\relax
\typeout{** WARNING: IEEEtran.bst: No hyphenation pattern has been}%
\typeout{** loaded for the language `#1'. Using the pattern for}%
\typeout{** the default language instead.}%
\else
\language=\csname l@#1\endcsname
\fi
#2}}
\providecommand{\BIBdecl}{\relax}
\BIBdecl

\bibitem{argall2009survey}
B.~D. Argall, S.~Chernova, M.~Veloso, and B.~Browning, ``A survey of robot
  learning from demonstration,'' \emph{Robotics and autonomous systems},
  vol.~57, no.~5, pp. 469--483, 2009.

\bibitem{Suomalainen2018}
M.~Suomalainen, J.~Koivum{\"a}ki, S.~Lampinen, J.~Mattila, and V.~Kyrki,
  ``Learning from demonstration for hydraulic manipulators,'' in \emph{IEEE/RSJ
  Int. Conf. Intell. Robots and Syst. (IROS)}, Oct 2018.

\bibitem{PervezRyuLfD}
A.~Pervez, A.~Ali, J.~Ryu, and D.~Lee, ``Novel learning from demonstration
  approach for repetitive teleoperation tasks,'' in \emph{IEEE World Haptics
  Conf. (WHC)}, June 2017, pp. 60--65.

\bibitem{LfDSurvey2018}
I.~Havoutis and S.~Calinon, ``Learning from demonstration for semi-autonomous
  teleoperation,'' \emph{Autonomous Robots}, vol.~43, no.~3, pp. 713--726, Mar
  2019.

\bibitem{Mattila2017}
J.~Mattila, J.~Koivum\"{a}ki, D.~G. Caldwell, and C.~Semini, ``A survey on
  control of hydraulic robotic manipulators with projection to future trends,''
  \emph{IEEE/ASME Trans. Mechatronics}, vol.~22, no.~2, pp. 669--680, 2017.

\bibitem{Koivumaki_TRO2015}
J.~Koivum\"{a}ki and J.~Mattila, ``Stability-guaranteed force-sensorless
  contact force/motion control of heavy-duty hydraulic manipulators,''
  \emph{IEEE Trans. Robot.}, vol.~31, no.~4, pp. 918--935, Aug 2015.

\bibitem{HOKAYEM20062035}
P.~F. Hokayem and M.~W. Spong, ``Bilateral teleoperation: An historical
  survey,'' \emph{Automatica}, vol.~42, no.~12, pp. 2035--2057, 2006.

\bibitem{OSTOJASTARZEWSKI1989}
M.~Ostoja-Starzewski and M.~Skibniewski, ``A master-slave manipulator for
  excavation and construction tasks,'' \emph{Rob Auton Syst}, vol.~4, no.~4,
  pp. 333--337, 1989.

\bibitem{Salcudean1999}
S.~Salcudean, K.~Hashtrudi-Zaad, S.~Tafazoli, S.~P. DiMaio, and C.~Reboulet,
  ``Bilateral matched-impedance teleoperation with application to excavator
  control,'' \emph{IEEE Control Systems}, vol.~19, no.~6, pp. 29--37, 1999.

\bibitem{Tafazoli2002}
S.~Tafazoli, S.~E. Salcudean, K.~Hashtrudi-Zaad, and P.~D. Lawrence,
  ``Impedance control of a teleoperated excavator,'' \emph{IEEE Trans. Control
  Syst. Technol.}, vol.~10, no.~3, pp. 355--367, May 2002.

\bibitem{KoivumakiTMECH2017}
J.~Koivum{\"a}ki and J.~Mattila, ``Stability-guaranteed impedance control of
  hydraulic robotic manipulators,'' \emph{IEEE/ASME Trans. Mechatronics},
  vol.~22, no.~2, pp. 601--612, 2017.

\bibitem{KoivumakiCEP2019}
J.~Koivum{\"a}ki, W.-H. Zhu, and J.~Mattila, ``Energy-efficient and
  high-precision control of hydraulic robots,'' \emph{Control Engineering
  Practice}, vol.~85, pp. 176--193, Aug. 2019.

\bibitem{ZhaiTeleop2018}
D.~Zhai and Y.~Xia, ``A novel switching-based control framework for improved
  task performance in teleoperation system with asymmetric time-varying
  delays,'' \emph{IEEE Trans. Cybern.}, vol.~48, no.~2, pp. 625--638, 2018.

\bibitem{Aijaz2017}
A.~Aijaz, A.~H. Aghvami, V.~Friderikos, and M.~Frodigh, ``Realizing the tactile
  internet: Haptic communications over next generation {5G} cellular
  networks,'' \emph{IEEE Trans. Wireless Commun.}, vol.~24, no.~2, pp. 82--89,
  April 2017.

\bibitem{Guo2018}
J.~Guo, C.~Liu, and P.~Poignet, ``A scaled bilateral teleoperation system for
  robotic-assisted surgery with time delay,'' \emph{Journal of Intelligent {\&}
  Robotic Systems}, Aug 2018.

\bibitem{MalyszSirouspour2}
P.~Malysz and S.~Sirouspour, ``A kinematic control framework for single-slave
  asymmetric teleoperation systems,'' \emph{IEEE Trans. Robot.}, vol.~27,
  no.~5, pp. 901--917, Oct 2011.

\bibitem{Sirouspour2009}
A.~Shahdi and S.~Sirouspour, ``Adaptive/robust control for time-delay
  teleoperation,'' \emph{IEEE Trans. Robot.}, vol.~25, no.~1, pp. 196--205, Feb
  2009.

\bibitem{Sirouspour2005}
S.~Sirouspour, ``Modeling and control of cooperative teleoperation systems,''
  \emph{IEEE Trans. Robot.}, vol.~21, no.~6, pp. 1220--1225, Dec 2005.

\bibitem{Muhammad2007}
A.~Muhammad, S.~Esque, J.~Mattila, M.~Tolonen, P.~Nieminen, O.~Linna,
  M.~Vlenius, M.~Siuko, J.~Palmer, and M.~Irving, ``Development of water
  hydraulic remote handling system for divertor maintenance of {ITER},'' in
  \emph{IEEE 22nd Symp. on Fusion Engineering}, June 2007, pp. 1--4.

\bibitem{Zhu2000}
W.-H. Zhu and S.~E. Salcudean, ``Stability guaranteed teleoperation: an
  adaptive motion/force control approach,'' \emph{IEEE Trans. Automatic
  Control}, vol.~45, no.~11, pp. 1951--1969, Nov 2000.

\bibitem{LampinenCASE}
S.~Lampinen, J.~Koivum\"{a}ki, and J.~Mattila, ``Full-dynamics-based bilateral
  teleoperation of hydraulic robotic manipulators,'' in \emph{IEEE 14th Int.
  Conf. Automation Science and Engineering}, Aug 2018, pp. 1343--1350.

\bibitem{LampinenFPMC2018}
S.~Lampinen, J.~Koivum{\"a}ki, and J.~Mattila, ``Bilateral teleoperation of a
  hydraulic robotic manipulator in contact with physical and virtual
  constraints,'' in \emph{BATH/ASME Symp. on Fluid Power and Motion
  Control}.\hskip 1em plus 0.5em minus 0.4em\relax ASME, 2018.

\bibitem{Zhu2010Virtual}
W.-H. Zhu, \emph{Virtual decomposition control: toward hyper degrees of freedom
  robots}.\hskip 1em plus 0.5em minus 0.4em\relax Springer Science \& Business
  Media, 2010, vol.~60.

\bibitem{Zhu1997}
W.-H. Zhu, Y.-G. Xi, Z.-J. Zhang, Z.~Bien, and J.~D. Schutter, ``Virtual
  decomposition based control for generalized high dimensional robotic systems
  with complicated structure,'' \emph{IEEE Trans. Robot. Autom.}, vol.~13,
  no.~3, pp. 411--436, Jun 1997.

\bibitem{Zhu2013}
W.-H. {Zhu}, T.~{Lamarche}, E.~{Dupuis}, D.~{Jameux}, P.~{Barnard}, and
  G.~{Liu}, ``Precision control of modular robot manipulators: The {VDC}
  approach with embedded {FPGA},'' \emph{IEEE Trans. Robot.}, vol.~29, no.~5,
  pp. 1162--1179, 2013.

\bibitem{Zhu2008dynamics}
W.-H. Zhu, ``Dynamics of general constrained robots derived from rigid
  bodies,'' \emph{J Appl Mech}, vol.~75, no.~3, pp. 031\,005--031\,005--11,
  2008.

\bibitem{Kazerooni1994}
H.~Kazerooni and M.-G. Her, ``The dynamics and control of a haptic~interface
  device,'' \emph{IEEE Trans. Robot. Autom.}, vol.~10, no.~4, pp. 453--464,
  1994.

\bibitem{cooke1979}
J.~D. Cooke, ``\BIBforeignlanguage{English}{Dependence of human arm movements
  on limb mechanical properties},'' \emph{\BIBforeignlanguage{English}{Brain
  Research}}, vol. 165, no.~2, pp. 366--369, 1979.

\bibitem{MalyszSirouspour}
P.~Malysz and S.~Sirouspour, ``Nonlinear and filtered force/position mappings
  in bilateral teleoperation with application to enhanced stiffness
  discrimination,'' \emph{IEEE Trans. Robot.}, vol.~25, no.~5, pp. 1134--1149,
  Oct 2009.

\bibitem{Haddadin2017}
S.~{Haddadin}, A.~{De Luca}, and A.~{Albu-Sch{\"a}ffer}, ``Robot collisions: A
  survey on detection, isolation, and identification,'' \emph{IEEE Transactions
  on Robotics}, vol.~33, no.~6, pp. 1292--1312, Dec 2017.

\bibitem{Koivumaki_ASME2017}
J.~Koivum{\"a}ki and J.~Mattila, ``Adaptive and nonlinear control of discharge
  pressure for variable displacement axial piston pumps,'' \emph{ASME J. Dyn.
  Syst., Meas., Control}, vol. 139, no.~10, 2017.

\bibitem{Tao1997}
G.~Tao, ``A simple alternative to the {Barbalat} lemma,'' \emph{IEEE Trans.
  Autom. Control}, vol.~42, no.~5, 1997.

\end{thebibliography}

\end{document}